\documentclass{article}

\usepackage{verbatim}
\usepackage[hyphens]{url}  
\usepackage{graphicx} 
\usepackage{wrapfig} 
\usepackage{natbib}  
\usepackage{caption} 
\usepackage[table]{xcolor}
\usepackage{algorithm}
\usepackage{algorithmic}
\usepackage{amsmath}
\usepackage{amsthm} 

\usepackage{newfloat}
\usepackage{listings}
\DeclareCaptionStyle{ruled}{labelfont=normalfont,labelsep=colon,strut=off} 
\floatstyle{ruled}
\newfloat{listing}{tb}{lst}{}
\floatname{listing}{Listing}

\usepackage{booktabs}

\usepackage{graphicx}
\usepackage{booktabs}
\usepackage{tcolorbox}
\usepackage{amsmath,amssymb}
\newtheorem{theorem}{Theorem}
\newtheorem{lemma}[theorem]{Lemma} 
\usepackage{booktabs}
\usepackage{multirow}
\usetikzlibrary{decorations.pathmorphing}
\usepackage{graphicx}
\usepackage{subcaption}
\usepackage{mathtools}
\usepackage{enumitem}
\usepackage{tikz}
\usetikzlibrary{decorations.pathmorphing, backgrounds, calc, shapes.arrows}
\definecolor{elegantBlue}{RGB}{65, 105, 225} 
\definecolor{elegantOrange}{RGB}{237, 125, 49} 
\newtheorem{proposition}{Proposition}

\usepackage{mathtools} 
\usepackage{booktabs} 
\usepackage{tikz} 

\usepackage[accsupp]{axessibility}  
\usepackage[preprint]{neurips_2026}

\usepackage[utf8]{inputenc} 
\usepackage[T1]{fontenc}    
\usepackage{hyperref}       
\usepackage{url}            
\usepackage{booktabs}       
\usepackage{amsfonts}       
\usepackage{nicefrac}       
\usepackage{microtype}      
\usepackage{xcolor}         

\title{LATS: L\'evy Adaptive Tree Sampling for Feedback-Driven Diverse Target Discovery}

\author{
    Binglin Ji$^{1}$\thanks{Equal contribution} ,~~~Anindya Sarkar$^{1}$\footnotemark[1] ,~~~~Hengchang Lu$^{1},~~~~$Lecheng Kong$^{1}$\\ ~~~~\textbf{Yixin Chen}$^{1}$, ~~~~\textbf{Yevgeniy Vorobeychik}$^{1}$ \\
    \texttt{\{binglin.j,~anindya,~yvorobeychik\}@wustl.edu,}\\${}^1$Department of CSE, Washington University in St.Louis, USA}

\begin{document}

\maketitle

\begin{abstract}
  While diffusion models excel at capturing complex data distributions, scientific discovery often requires steering generation toward specific, uncharacterized regions that maximize a target objective. These high-utility modes frequently reside in low-likelihood tail regions and are only revealed sequentially through interactive feedback. Existing diffusion samplers fail in this regime: they inherit the pre-trained model’s bias toward high-density regions, leaving rare yet promising phenomena underexplored.  Conversely, exploration-heavy samplers ensure broad coverage but fail to efficiently exploit high-utility modes when constrained by a strict sampling budget.
To resolve this dilemma, we introduce Lévy Adaptive Tree Search (LATS), a principled sampling framework for online feedback-driven search. LATS leverages heavy-tailed exploration coupled with tree-based value backpropagation to progressively uncover preferred modes. By maintaining broad distributional coverage, LATS successfully discovers low-likelihood, high-utility regions while preserving sample fidelity and structural diversity. Experiments across diverse benchmarks, including materials science, demonstrate that LATS significantly outperforms baselines in target discovery efficiency.
\end{abstract}

\section{Introduction}\label{sec:intro}
Diffusion models have emerged as powerful generative engines capable of capturing the immense complexity of natural data manifolds by reversing a diffusion process. Yet, for scientific applications, the true goal is to actively discover and generate target samples from low-likelihood, high-utility regions of the data space—targets that are unknown a priori and revealed only through sequential feedback. In domains like material discovery, the ultimate objective is to guide generation toward compounds with a specific chemical property.
This can be framed as a reinforcement learning (RL) problem, where the objective is to fine-tune the diffusion model to maximize a reward that reflects the desired properties of the targets~\citep{uehara2024feedback}. However, these methods hinge on an online-trained reward model whose early-phase bias can misdirect the fine-tuning process and degrade sample quality. Moreover, because the fine-tuning objective effectively minimizes a reverse-KL against a reward-tilted distribution~\citep{kim2025test}—an inherently mode-seeking loss—it ignores many high-utility modes and collapses onto a narrow band of high-reward regions of the underlying distribution. 
\begin{figure*}[!h]
    \centering
    \includegraphics[width=1.0\textwidth, trim=10 0 35 0,
clip]{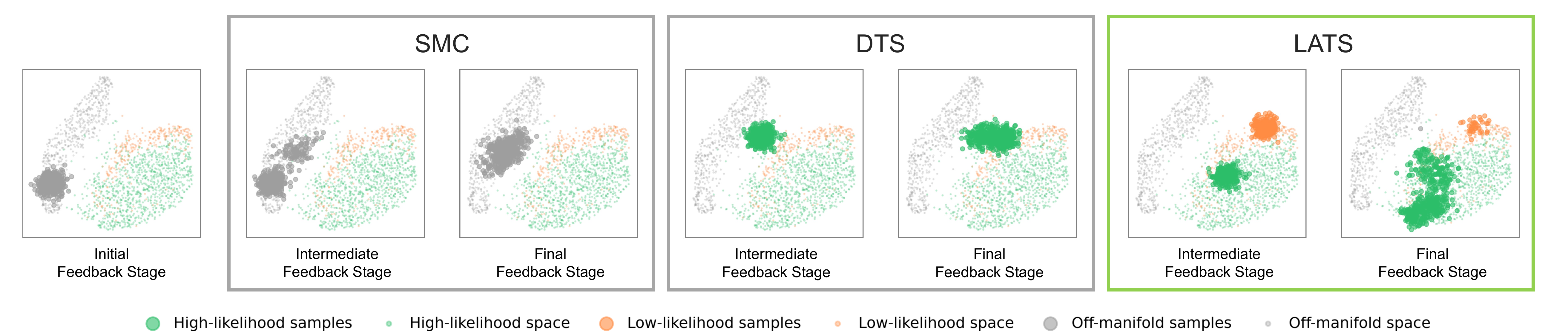}
    \vspace{-10pt}
    \caption{\small{Unlike SMC (which suffers from noisy reward gradient guidance) and DTS (which is constrained by Gaussian bias), LATS discovers high-utility tail modes (Viz., using MNIST digits \textcolor[HTML]{2FA56E}{7} and \textcolor{orange}{9} as targets, which are the two least represented categories during training) while avoiding \textcolor{darkgray}{\textbf{off-manifold}} sample generation. LATS leverages online feedback to enrich intra-target class diversity actively, thereby avoiding mode collapse.}}
    \label{fig:overview}
\end{figure*}

Recent SMC–based inference-time scaling methods~\citep{singhal2025general,skreta2025feynman,kim2025test} circumvent direct model fine-tuning by operating on fixed generative priors. However, they exhibit severe weight degeneracy, which effectively collapses the proposals and suppresses sample diversity~\citep{lee2025debiasing}. Moreover, these approaches rely on reward values or gradients at every resampling step; in practice, these signals are unavailable and must be approximated, introducing systematic bias into the sampler and degrading sample quality. To mitigate these issues, tree-based inference-time samplers~\citep{guo2025training,jain2025diffusion,ji2026bootstrap} have been introduced that backpropagate only terminal rewards, eliminating the need to evaluate rewards at every intermediate denoising step. Yet their search remains anchored to high-likelihood regions under the base diffusion prior, making it difficult to discover diverse target modes—particularly low-likelihood ones.
Consequently, they fall short of generating truly diverse, target-aligned samples. We aim to tackle the limitations of prior approaches and derive a sampling framework with the following objective:
\begin{tcolorbox}[colback=gray!5!white, colframe=blue!40!black, boxrule=0.5pt, arc=2mm]
Given a long-tailed dataset and targets with unique, heterogeneous properties—often residing in the tail and revealed only through limited interactive feedback—design a sampler that produces diverse target-aligned samples while preserving high fidelity.
\end{tcolorbox}
Achieving this objective requires \textbf{efficient exploration}.
In particular, in high-dimensional spaces (e.g., images), exploration is not merely about visiting new regions but about navigating the low-dimensional data manifold while respecting the problem’s structural constraints. Realistic solutions typically occupy a thin manifold inside a vast design space.
An effective exploration scheme must traverse this manifold strategically—maintaining sample quality and enriching diversity within the target set—while avoiding off-manifold regions that correspond to invalid, wasteful queries. However, if the sampler persistently explores without generating samples that meet the diverse target properties, it may fail to discover enough high-value instances. The model must therefore carefully balance exploration with \textbf{exploitation}—effectively generating target-aligned samples while minimizing costly online reward queries.

Lévy process~\cite{shariatian2024heavy} has been proposed to generalize the DDPM framework by injecting heavy-tailed $\alpha$-stable noise, yielding a Lévy diffusion model whose large jumps improve coverage of compactly supported, class-imbalanced data and help uncover weakly represented modes. This model generates diverse denoising trajectories that better cover long-tail regions.
However, our goal is not merely coverage: we seek to concentrate probability mass on high-value modes while suppressing others. 
To this end, we propose a simple but effective framework for online feedback-driven diverse target discovery that leverages the complementary strengths of two mechanisms. While the Lévy diffusion model enables the exploration of diverse modes, the tree-based sampler provides critical exploitation by pruning branches unlikely to yield target-aligned samples, thereby preventing non-preferred generations.
We refer to the proposed sampling framework as \textbf{L}\'evy \textbf{A}daptive \textbf{T}ree \textbf{S}earch (LATS). Fig.~\ref{fig:overview} demonstrates the superiority of LATS over baseline methods in discovering targets spanning across diverse classes, including low-likelihood modes.
We summarize the contributions below:
\begin{itemize}[noitemsep,topsep=0pt, leftmargin=*]
    \item We introduce LATS, a simple yet effective sampling framework for feedback-efficient diverse target discovery. 
    \item We validate the effectiveness of each component of LATS through comprehensive quantitative and qualitative analysis across diverse settings, including material science.
    \item We show that LATS can generate diverse, target-aligned samples, including those from long-tail preference modes, without sacrificing sample quality. 
\end{itemize}





\vspace{-6pt}
\section{Related Work}\label{sec:rel_work}
\smallskip
\vspace{-6pt}
\paragraph{Fine-Tune Diffusion model with RL:}
Fine-tuning generative models with human feedback, such as user preferences, has become increasingly prevalent~\citep{ouyang2022training,touvron2023llama}.
Prior work on diffusion models has largely focused on optimizing reward functions via supervised learning~\citep{lee2023aligning,wu2023better}, control-based methods~\citep{xu2024imagereward,ajay2022conditional,janner2022planning}, or policy-gradient techniques~\citep{black2023training}. However, these approaches typically assume a static reward model, treating rewards as fixed ground truth and not accommodating online queries. In contrast, we study an online, interactive setting in which target properties are initially unknown and are gradually revealed through sequential feedback, enabling continual target discovery rather than one-shot alignment to a fixed reward proxy.
\citep{dong2023raft} introduces a general online learning framework for aligning generative models, but it is not tailored to diffusion models. More recently,  ~\citep{uehara2024feedback} proposed a feedback-efficient online fine-tuning method for diffusion models; however, it learns a separate online reward model to guide sampling, which amplifies bias in the early stages of fine-tuning, and due to the KL-regularized fine-tuning objective, it remains vulnerable to mode collapse. 
Moreover, these prior approaches are not tailored to sampling from long-tail high-utility classes and are prone to mode collapse when confronted with real-world long-tailed distributions.
\vspace{-4pt}
\paragraph{Inference-time Search Approaches:}
Recently, there has been a growing interest in using inference-time scaling approaches. Particle-based SMC methods are one of the most widely applicable approaches among them, where a population of samples is maintained to approximately sample from the desired distribution. SMC~\citep{del2006sequential} uses potential functions, which usually approximate the soft value function, to assign weights to particles and resample them at every step. Different variations of SMC have been proposed for diffusion model alignment~\citep{dou2024diffusion,kim2025test,trippe2022diffusion,wu2023practical}. Classical SMC guarantees exact sampling in the limit of infinite particles and exact value estimation. In practice, however, the repeated sampling procedure can reduce diversity due to weight variance and inaccurate value estimates. Moreover, these methods also rely on a learned reward model to compute particle weights for resampling. In an online setting like ours, the reward model is poorly calibrated in the early interaction phases, and this bias propagates into the resampling step, which in turn misguides the reverse denoising process and degrades the quality of the generated samples. While recent tree-based inference-time samplers~\citep{jain2025diffusion,guo2026training,ji2026bootstrap} bypass intermediate reward evaluations by backpropagating only terminal rewards, they remain severely constrained by a Gaussian bias that impedes sampling from preferred long-tail modes.
Recently, several correction techniques~\citep{skreta2025feynman,lee2025debiasing} have been proposed to alleviate particle weight degeneracy and debias particle-based guidance. However, these methods break down in high-dimensional settings because they require reward gradients at every resampling step—quantities that are not directly available and must be approximated—thereby introducing systematic bias into the sampling process. Additionally, these methods are not explicitly designed to sample effectively from long-tailed distributions—a central challenge in many real-world applications.

\section{Problem Formulation}\label{sec:problem}
Real-world data are typically long-tailed, with most probability mass concentrated on a few frequent patterns and a vast number of rare ones. We model this by assuming access to a long-tail dataset $\mathcal{D}_{lt}=\{ x_i\}^m_{i=1}$, consisting of $m$ i.i.d. samples drawn from an unknown distribution $\mathcal{P}_X$ over a data space $X$. Desired target properties in this setting are inherently heterogeneous and multimodal in nature:  a given downstream objective may favor several distinct modes of $\mathcal{P}_X$, including modes residing in the long-tail region. These high-value target properties are not known a priori and are revealed only through costly interactive feedback—for example, via wet-lab experiments in material discovery—so feedback queries must be used sparingly. Our goal is therefore to design a sampler that, given a fixed feedback budget $\mathcal{B}$, efficiently allocates queries to learn the underlying multi-modal target structure and generates samples that maximally align with the discovered target properties under the long-tailed data distribution. We formulate the objective as:
\vspace{-6pt}
\begin{equation}
    \max_{\{g_{\theta^{\mathcal{D}_{lt}}_t}\}}\sum_{t=1}^{t=\mathcal{B}}\mathbb{E}_{x_t \sim g_{\theta^{\mathcal{D}_{lt}}_t}}[r(x_t)],
\end{equation}
\vspace{-1pt}
Where, $g_{\theta^{\mathcal{D}_{lt}}_t}$ represents the underlying generative model (e.g., a diffusion model), parametrized by $\theta^{\mathcal{D}_{lt}}_t$, that is used to sample at the $t$-th interaction step. $\theta^{\mathcal{D}_{lt}}_0$ denotes the parameters of the initial generative model, obtained by training solely on the long-tail dataset $\mathcal{D}_{lt}$ and used to generate samples before any ground-truth feedback is observed. Direct fine-tuning–based methods update the generative model parameters after every feedback step, whereas our approach keeps the model fixed throughout the entire interaction; in particular, $\theta^{\mathcal{D}_{lt}}_0$ remains unchanged for all $t \in \{1, \ldots, \mathcal{B}\}$.
Here $r(.)$ denotes the (unknown) probabilistic reward function with values in $[0, 1]$, where $r(x_t) = 1$ indicates the generated sample at step $t$ satisfies the target property from the set of high-value target properties, and $r(x_t) = 0$ otherwise.
By leveraging the generative model trained on $\mathcal{D}_{lt}$, we aim to efficiently explore the sample space $X$ to uncover diverse high-utility target properties, while also exploiting the gathered information to generate samples aligning with those discovered properties. 

\section{Background}\label{sec:prelim}
\noindent\textit{$\alpha$-Stable Distributions and L\'evy Process:}
The $\alpha$-stable distributions form a rich class of heavy-tailed probability distributions that generalize the Gaussian distribution. While the Gaussian law governs the limiting behavior of sums of independent, identically distributed (i.i.d.) random variables with finite variance (via the Central Limit Theorem), $\alpha$-stable distributions serve as the limit for sums of i.i.d. variables with infinite variance (via the Generalized Central Limit Theorem). This property makes them naturally suited for modeling data with heavy tails, outliers, or impulsive characteristics. The Characteristic Function of $\alpha$-stable distribution is defined as follows: A random variable $X$ follows an $\alpha$-stable distribution, denoted $X \sim S_{\alpha,\beta}(\mu, \sigma)$, if its characteristic function $\phi(u) = \mathbb{E}[e^{iuX}]$ takes the form:$$\phi(u) = \exp\left( iu\mu - |\sigma u|^\alpha \left( 1 - i\beta \text{sgn}(u) \Phi(u, \alpha) \right) \right)$$where $\alpha \in (0, 2]$ is the stability index (or tail index), $\beta \in [-1, 1]$ is the skewness, $\sigma > 0$ is the scale, and $\mu \in \mathbb{R}$ is the location parameter. The term $\Phi(u, \alpha)$ is given by $\tan(\pi\alpha/2)$ if $\alpha \neq 1$ and $-(2/\pi)\log|u\sigma|$ if $\alpha = 1$. 
The parameter $\alpha$ controls the "heaviness" of the tails. For $\alpha = 2$, the distribution reduces to a Gaussian distribution (with variance $2\sigma^2$), which has light exponential tails. For $\alpha < 2$, the distribution exhibits heavy tails that decay asymptotically as a power law, $|x|^{-(1+\alpha)}$, leading to infinite variance. This infinite variance prevents the use of standard $L^2$-based analysis tools common in Gaussian diffusion models. 


\noindent\textit{Isotropic $\alpha$-stable distribution and relation to Gaussian Noise:} The random variable $X \in \mathbb{R}^d$ is isotropic $\alpha$-stable if its characteristic function is given by: for all $u \in \mathbb{R}^d$,
\[
\mathbb{E}\big[\exp(i u^\top X)\big]
= \exp\big(i \mu^\top u - \sigma^\alpha \lVert u \rVert^\alpha\big),
\]
where $\mu \in \mathbb{R}^d$ is the location parameter and $\sigma I_d$ plays the role of a covariance matrix. We denote it by $X \sim S^i_{\alpha}(\mu, \sigma I_d)$.
Next, we present an important result that establishes the fundamental connection between heavy‑tailed $\alpha-$stable noise and Gaussian noise with a random scale.
\begin{lemma}\label{eq:rel}[Relation Between $\alpha$-Stable and Gaussian Noise~\citep{samorodnitsky1994stable}]
Let $\alpha < 2$, and let $X \sim S^{i}_{\alpha}(\mu,\sigma I_d)$.
Then
\[
  X \overset{d}{=} \mu + \sigma A^{1/2} G,
\]
where $\overset{d}{=}$ denotes equality in distribution, $A \sim
S_{\alpha/2,1}(0,c_A)$ is a one-dimensional positive stable random
variable with $c_A := \cos^{2/\alpha}(\pi\alpha/4)$, and
$G \sim \mathcal{N}(0,I_d)$.
\end{lemma}

\section{Methodology}\label{sec:method}
Online feedback-efficient diverse target discovery fundamentally hinges on a generative model that can efficiently cover all modes of a long-tail distribution, aiding in discovering and generating samples with rare, idiosyncratic target properties. 
Standard diffusion-based generators are fundamentally misaligned with this regime: their forward and reverse dynamics are both built on Gaussian transition kernels, imposing a light‑tailed geometry that inherently under-represents long‑tail regions of the underlying distribution. As training unfolds under this Gaussian bias, the model is gradually driven toward dominant target modes, eroding coverage of diverse, low-mass modes. The following result characterizes this inherent bias in vanilla Gaussian diffusion when trained with long-tail distributions.
\begin{proposition}
Let the data space $\mathcal{X}$ be partitioned into majority classes $\mathcal{C}_{maj}$ and tail classes $\mathcal{C}_{tail}$ (i.e. $\mathcal{C} = \mathcal{C}_{maj} \cup \mathcal{C}_{tail}$ ), such that the initial training distribution $q(x)$ satisfies $P(\mathcal{C}_{maj}) \gg P(\mathcal{C}_{tail})$ (i.e., representing a class-imbalance long-tail dataset $\mathcal{D}_{lt}$). Also, let $w_c$ represent the portion of training images for class $c \in \mathcal{C}$ such that $\sum_{c \in \mathcal{C}} w_c = 1$. Note that $\sum_{c \in \mathcal{C}_{maj}} w_c \gg \sum_{c \in \mathcal{C}_{tail}} w_c$. The original training objective of DDPM~\citep{ho2020denoising} can then be decomposed into a weighted sum of class-specific KL divergences:
\begin{equation}\label{eq:im}
\small{
\mathbb{E}_{q} \left[ \sum_{t\geq1}^T D_{\text{KL}}[q(x_{t-1}|x_t,x_0,c) \parallel p_{\theta}( x_{t-1}|x_t,c)] \right]
= \sum_{c \in \mathcal{C}} w_c \mathbb{E}_{q} \left[ \sum_{t \geq 1}^T D_{\text{KL}} [q(x_{t-1} \mid x_t, x_0, c) \parallel p_\theta(x_{t-1} \mid x_t, c)] \right].}
\end{equation}
\end{proposition}
We provide the proof in the Appendix. According to Eqn.~\ref{eq:im}, DDPM is evidently biased toward head classes with large $w_c$, which undermines the efficient exploration key to diverse target discovery. 
To overcome this limitation, we employ the Lévy diffusion process~\cite{shariatian2024heavy}, which replaces Gaussian noise with $\alpha$-stable noise in the forward process. This substitution exploits the heavy-tailed nature of $\alpha$-stable distributions to better capture long-tail structures within the denoising probabilistic framework.
Accordingly, the forward noise process can be modeled as a Markov chain $X_0 \sim\mathcal{D}_{lt}$, and for $t \in \{1, \dots, T\}$,
\begin{equation} \label{eq:5}
    X_t = \gamma_t X_{t-1} + \sigma_t A_t^{1/2} G_t ,
\end{equation}
where the sequences $\{G_t\}_{t=1}^T$ and $\{A_t\}_{t=1}^T$ are independent, with $G_t \sim \mathcal{N}(0, I_d)$, and $A_t \sim S_{\alpha/2, 1}(0, c_A)$, and $c_A = \cos^{2/\alpha}(\pi\alpha/4)$. In view of Lemma~\ref{eq:rel}, conditioned on $\{A_t\}_{t=1}^T$ and $X_0$, it yields a Markov chain $\{X_t\}_{t=1}^T$ with Gaussian transition densities:
\begin{equation}
    \mathcal{T}^{(\alpha)}_{1:T|0,a}(x_{1:T} | x_0, a_{1:T}) = \prod_{t=1}^T \mathcal{N}_d(x_t | \gamma_t x_{t-1}, \sigma_t^2 a_t) ,
\end{equation}

where $\mathcal{N}_d(x\mid m, \Sigma)$ denotes the density of the $d$-dimensional Gaussian distribution with mean $m$ and covariance matrix $\Sigma$. Following~\citep {shariatian2024heavy}, we define the backward diffusion process as: 
\begin{equation}
\mathcal{T}^{(\alpha)}_{1:T \mid 0,a}\bigl(x_{1:T} \mid x_0, a_{1:T}\bigr)
:= p^{(\alpha)}_T(x_T) \prod_{t=2}^T
\mathcal{T}^{(\alpha)}_{t-1 \mid 0,t,a}\bigl(x_{t-1} \mid x_t, x_0, a_{1:t}\bigr) ,
\end{equation}
where $\mathcal{T}^{(\alpha)}_{t-1 \mid 0,t,a}\bigl(x_{t-1} \mid x_t, x_0, a_{1:t})$ is now the tractable density of a Gaussian distribution, yielding a simplified denoising training objective analogous to DDPM. Concretely, the training objective is:
\begin{equation} \label{eq:8_simple}
    \mathcal{L}_{\text{DLM}}(\theta) = \sum_{t=1}^T \mathbb{E} \left[ \mathbb{E} \left[ \| \hat{\epsilon}^\theta_t(X_t) - \epsilon_t(X_t, X_0) \|^2 \;\middle|\; A_{1:t} \right]^{1/2} \right] ,
\end{equation}
where the model $\hat{\epsilon}^\theta_t$ is designed to fit the noise $\epsilon_t(X_t, X_0) = (X_t - \gamma_{1\to t}X_0)/\sigma_{1\to t}$ added at time-step $t$, and $\gamma_{1\to t} := \prod_{i=1}^t \gamma_i, \quad \text{and} \quad \sigma_{1\to t} := \left( \sum_{i=1}^t \left( \frac{\gamma_{1\to t} \sigma_i}{\gamma_{1\to i}} \right)^\alpha \right)^{1/\alpha}$. The detailed derivation of Eqn.~\ref{eq:8_simple} is provided in~\cite{shariatian2024heavy}. Next, we present a Proposition that justifies the rationale behind using $\alpha$-stable distribution to encourage exploration across different modes of the underlying distribution.
\begin{proposition}\label{th:levy-collapse} Persistent Gradient Signal in Heavy-Tailed Diffusion. Let $\mathcal{X}$ be a data space containing a majority mode at $x=0$ and a minority mode at $x=\mu$. Let $s(x, t) = \nabla \log p_t(x)$ be the score function at diffusion time $t$. Then, (1) \emph{Gaussian Collapse (DDPM):} For Gaussian noise kernels, the signal from the minority mode is exponentially suppressed: $p_t^G(\mu) \sim \mathcal{O}(e^{-\|\mu\|^2})$. Consequently, the gradient signal vanishes for large $\mu$, leading to mode collapse; (2) \emph{Lévy Mitigation (LM):} For $\alpha$-stable noise kernels ($0 < \alpha < 2$), the minority mode maintains an algebraic relationship with the majority mode: $p_t^L(\mu) \sim \mathcal{O}(\|\mu\|^{-(1+\alpha)})$.
\end{proposition}

Hence, according to Proposition~\ref{th:levy-collapse}, the heavy-tailed kernel ensures that the training objective remains sensitive to minority modes across long distances, preventing the optimizer from ignoring rare classes in imbalanced datasets. 

So far, we have focused on learning a generative model that can aggressively explore diverse modes, including the long tail, which is essential for efficiently uncovering diverse properties of the latent target landscape. However, pure exploration is not sufficient, since our ultimate goal is to produce as many target-aligned samples as possible under a strict sampling budget. A naive option is to online fine-tune the L\'evy diffusion model trained on $\mathcal{D}_{lt}$ using streaming feedback, but this leads to inference-time mode collapse. The following theorem formalizes this phenomenon.
\begin{theorem}\label{th:it-collapse}
 Assuming the weighting function $w(x)$ (derived from the reward $r(x)$) is positively correlated with data distribution generated by the trained L\'evy Diffusion model $g^0_{\theta^{\mathcal{D}_{lt}}_0}(x)$, satisfying the following relation:
\begin{equation}
    P(x \in \mathcal{C}_{maj}) \gg P(x \in \mathcal{C}_{tail}) \quad \text{where } x \sim g^0_{\theta^{\mathcal{D}_{lt}}_0}(x)
\end{equation}
Then, as the number of online fine-tuning epochs $N \to \infty$, the generated distribution $g^N_{\theta^{\mathcal{D}_{lt}}_0}(x)$ converges to a distribution supported entirely on the majority class $\mathcal{C}_{maj}$, and the probability of generating samples from the tail class $\mathcal{C}_{tail}$ decays to zero exponentially. More concretely:
\begin{equation}\label{eq:mc}
\text{Ratio}^N = \frac{P^N(\mathcal{C}_{maj})}{P^N(\mathcal{C}_{tail})} = \frac{\int_{\mathcal{C}_{maj}} g^N_{\theta^{\mathcal{D}_{lt}}_0}(x) dx}{\int_{\mathcal{C}_{tail}} g^N_{\theta^{\mathcal{D}_{lt}}_0}(x) dx} \to \infty
\end{equation}
\end{theorem}

To overcome this bottleneck, we pair an online adaptive tree sampler with the trained Lévy model.
We overcome this via an online adaptive tree sampler paired with the trained Lévy model. The resulting sampler aggressively prunes non-preferred trajectories explored via the trained Lévy model while boosting high-value modes—all without fine-tuning the trained L\'evy diffusion model. 

Specifically, following the standard tree sampler~\cite{jain2025diffusion}, we construct a tree, where nodes represent states $x_t$, and edges represent transitions $p_{\theta_{0}^{\mathcal{D}_{lt}}}(x_{t-1} \mid x_t)$ following the trained L\'evy diffusion model. Note that the Markov property of the reverse L\'evy diffusion chain naturally induces a finite horizon tree in $\mathbb{R}^d$, where $d$ is the dimensionality of the space over which we are diffusing.
Each node maintains the current state and timestep $(x_t, t)$, the visit count $N(x_t)$, and a Monte Carlo estimate of the soft value function $\hat{v}(x_t)$ that satisfies the soft Bellman equation, which is defined as: 
\begin{equation}\label{eq:value}\hat{v}_t(x_t) = \frac{1}{\lambda} \log \mathbb{E}_{p_{\theta_{0}^{\mathcal{D}_{lt}}}(x_{t-1} \mid x_t)} \left[ \exp (\lambda \hat{v}_{t-1}(x_{t-1})) \right].
\end{equation}

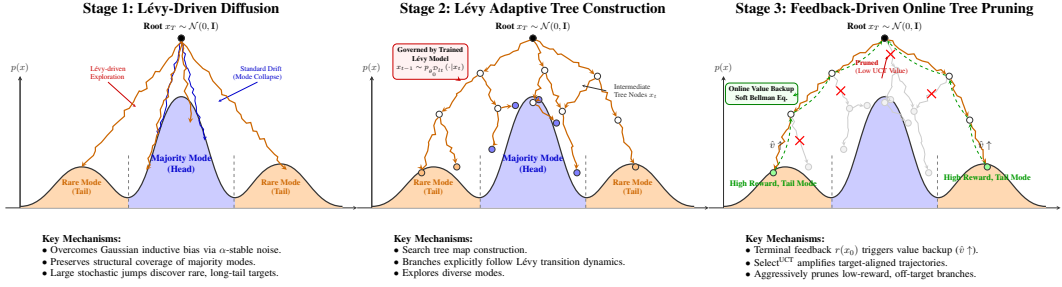
\begin{figure*}[t]
\centering
\resizebox{\textwidth}{!}{%
\begin{tikzpicture}[
    >=stealth,
    standard path/.style={thick, ->, draw=blue!80!black, decorate, decoration={random steps, segment length=4pt, amplitude=1.5pt}},
    active edge/.style={very thick, ->, draw=orange!80!black, decorate, decoration={random steps, segment length=5pt, amplitude=2.5pt}},
    pruned edge/.style={thick, ->, draw=gray!40, decorate, decoration={random steps, segment length=5pt, amplitude=2.5pt}},
    backup arrow/.style={thick, ->, dashed, draw=green!60!black},
    active node/.style={circle, draw=black!80, fill=white, thick, inner sep=1.5pt, minimum size=6pt},
    pruned node/.style={circle, draw=gray!40, fill=gray!10, thick, inner sep=1.5pt, minimum size=6pt}
]

\def\continuousMultimodal{
    (-5.5, 0.05) 
    .. controls (-4.5, 0.05) and (-4.2, 1.4) .. (-3.4, 1.4) 
    .. controls (-2.6, 1.4) and (-2.3, 0.3) .. (-1.8, 0.35) 
    .. controls (-1.0, 0.45) and (-0.8, 3.8) .. (0, 3.8) 
    .. controls (0.8, 3.8) and (1.0, 0.45) .. (1.8, 0.35) 
    .. controls (2.3, 0.3) and (2.6, 1.5) .. (3.4, 1.5) 
    .. controls (4.2, 1.5) and (4.5, 0.05) .. (5.5, 0.05) 
    -- (5.5, 0) -- (-5.5, 0) -- cycle
}

\newcommand{\drawManifold}{
    \draw[->, thick, black!80] (-6, 0) -- (6, 0);
    \draw[->, thick, black!80] (-5.5, 0) -- (-5.5, 4.5);
    \fill[orange!30] \continuousMultimodal;
    \begin{scope}
        \clip (-1.8, 0) rectangle (1.8, 4);
        \fill[blue!20] \continuousMultimodal;
    \end{scope}
    \draw[very thick, black!80] 
        (-5.5, 0.05) .. controls (-4.5, 0.05) and (-4.2, 1.4) .. (-3.4, 1.4) 
        .. controls (-2.6, 1.4) and (-2.3, 0.3) .. (-1.8, 0.35) 
        .. controls (-1.0, 0.45) and (-0.8, 3.8) .. (0, 3.8) 
        .. controls (0.8, 3.8) and (1.0, 0.45) .. (1.8, 0.35) 
        .. controls (2.3, 0.3) and (2.6, 1.5) .. (3.4, 1.5) 
        .. controls (4.2, 1.5) and (4.5, 0.05) .. (5.5, 0.05);
    \draw[dashed, thick, black!50] (-1.8, 0) -- (-1.8, 1.8);
    \draw[dashed, thick, black!50] (1.8, 0) -- (1.8, 1.8);
    \node[active node, fill=black] (root) at (0, 5.8) {};
}

\begin{scope}[xshift=0cm]
    \node[font=\Large\bfseries] at (0, 6.8) {Stage 1: Lévy-Driven Diffusion};
    \drawManifold
\node[active node, fill=black, label={[font=\small\bfseries]above:Root $x_T \sim \mathcal{N}(0, \mathbf{I})$}] (root) at (0, 5.8) {};

\node[align=left, font=\scriptsize, color=red!80!black] at (-2.5, 4.6) {Lévy-driven \\ Exploration};
\draw[->, red!80!black, shorten >=2pt] (-2.5, 4.3) -- (-1.5, 3.4);

\node[align=right, font=\scriptsize, color=blue!80!black] at (2.6, 4.6) {Standard Drift \\ (Mode Collapse)};
\draw[->, blue!80!black, shorten >=2pt] (2.6, 4.3) -- (1.2, 3.6);

\draw[->, thick, black!80] (-5.5, 0) -- (-5.5, 4.5) node[above, font=\small] {$p(x)$};

\node[align=center, font=\footnotesize\bfseries, text=orange!80!black] at (-3.4, 0.7) {Rare Mode \\ (Tail)};
\node[align=center, font=\small\bfseries, text=blue!80!black] at (0, 1.5) {Majority Mode \\ (Head)};
\node[align=center, font=\footnotesize\bfseries, text=orange!80!black] at (3.4, 0.7) {Rare Mode \\ (Tail)};

    \draw[standard path] (root.south) -- (-0.8, 2.5);
    \draw[standard path] (root.south) -- (0.8, 2.3);

    \draw[active edge] (root.south) -- (-1.2, 1.3); 
    \draw[active edge] (root.south) -- (0.3, 2.9);  
    \draw[active edge] (root.south) -- (-3.4, 1.4); 
    \draw[active edge] (root.south) -- (3.6, 1.3);  

    \node[anchor=north, align=left, font=\normalsize, text width=9.5cm] at (0, -0.8) {
        \textbf{Key Mechanisms:}\\
        $\bullet$ Overcomes Gaussian inductive bias via $\alpha$-stable noise.\\
        $\bullet$ Preserves structural coverage of majority modes.\\
        $\bullet$ Large stochastic jumps discover rare, long-tail targets.
    };
\end{scope}

\begin{scope}[xshift=12cm]
    \node[font=\Large\bfseries] at (0, 6.8) {Stage 2: L\'evy Adaptive Tree Construction};
    \drawManifold
\node[active node, fill=black, label={[font=\small\bfseries]above:Root $x_T \sim \mathcal{N}(0, \mathbf{I})$}] (root) at (0, 5.8) {};

\node[align=center, font=\footnotesize\bfseries, text=orange!80!black] at (-3.4, 0.7) {Rare Mode \\ (Tail)};
\node[align=center, font=\small\bfseries, text=blue!80!black] at (0, 1.5) {Majority Mode \\ (Head)};
\node[align=center, font=\footnotesize\bfseries, text=orange!80!black] at (3.4, 0.7) {Rare Mode \\ (Tail)};

    \node[active node] (d1_1) at (-1.8, 4.6) {}; \node[active node] (d1_2) at (0.4, 4.7) {}; \node[active node] (d1_3) at (2.2, 4.5) {};
    \node[active node] (d2_1) at (-3.2, 3.2) {}; \node[active node] (d2_2) at (-1.2, 3.4) {}; \node[active node] (d2_3) at (0.0, 3.6) {}; \node[active node] (d2_4) at (1.0, 3.3) {}; \node[active node] (d2_5) at (2.9, 3.0) {};
    
    \node[active node, fill=orange!50] (leaf_1) at (-3.8, 1.2) {}; \node[active node, fill=orange!50] (leaf_2) at (-2.6, 1.4) {}; 
    \node[active node, fill=blue!50] (leaf_3) at (-1.4, 2.0) {}; \node[active node, fill=blue!50] (leaf_4) at (-0.6, 3.5) {}; 
    \node[active node, fill=blue!50] (leaf_5) at (0.2, 3.7) {}; \node[active node, fill=blue!50] (leaf_6) at (0.8, 2.9) {}; 
    \node[active node, fill=blue!50] (leaf_7) at (1.5, 1.2) {}; \node[active node, fill=orange!50] (leaf_8) at (3.5, 1.4) {};

    \draw[active edge] (root) -- (d1_1); \draw[active edge] (root) -- (d1_2); \draw[active edge] (root) -- (d1_3);
    \draw[active edge] (d1_1) -- (d2_1); \draw[active edge] (d1_1) -- (d2_2); \draw[active edge] (d1_2) -- (d2_3); \draw[active edge] (d1_2) -- (d2_4); \draw[active edge] (d1_3) -- (d2_5); \draw[active edge] (d1_3) -- (d2_4);
    \draw[active edge] (d2_1) -- (leaf_1); \draw[active edge] (d2_1) -- (leaf_2); \draw[active edge] (d2_2) -- (leaf_3); \draw[active edge] (d2_2) -- (leaf_4); \draw[active edge] (d2_3) -- (leaf_5); \draw[active edge] (d2_3) -- (leaf_6); \draw[active edge] (d2_4) -- (leaf_7); \draw[active edge] (d2_5) -- (leaf_8);

\node[
    draw=red!80!black, 
    fill=red!5, 
    rounded corners, 
    thick, 
    font=\scriptsize\bfseries, 
    align=center
] (annotation) at (-3.5, 5.0) {Governed by Trained\\Lévy Model\\$x_{t-1} \sim p_{\theta^{\mathcal{D}_{lt}}_{0}}(\cdot|x_t)$};

\draw[->, thick, black!80] (-5.5, 0) -- (-5.5, 4.5) node[above, font=\small] {$p(x)$};

\draw[->, very thick, red!80!black, shorten >=2pt] (annotation) -- (-2.4, 4.3);

\node[align=left, font=\scriptsize, color=black!80] at (3.5, 4.0) {Intermediate\\Tree Nodes $x_t$};
\draw[->, black!80, shorten >=2pt] (2.5, 4.0) -- (1.6, 4.2);
    
    \node[anchor=north, align=left, font=\normalsize, text width=9.5cm] at (0, -0.8) {
        \textbf{Key Mechanisms:}\\
        $\bullet$ Search tree map construction.\\
        $\bullet$ Branches explicitly follow Lévy transition dynamics.\\
        $\bullet$ Explores diverse modes.
    };
\end{scope}

\begin{scope}[xshift=24cm]
    \node[font=\Large\bfseries] at (0, 6.8) {Stage 3: Feedback-Driven Online Tree Pruning};
    \drawManifold
    \node[active node, fill=black, label={[font=\small\bfseries]above:Root $x_T \sim \mathcal{N}(0, \mathbf{I})$}] (root) at (0, 5.8) {};

    \node[active node] (d1_1) at (-1.8, 4.6) {}; \node[pruned node] (d1_2) at (0.4, 4.7) {}; \node[active node] (d1_3) at (2.2, 4.5) {};
    \node[active node] (d2_1) at (-3.2, 3.2) {}; \node[pruned node] (d2_2) at (-1.2, 3.4) {}; \node[pruned node] (d2_3) at (0.0, 3.6) {}; \node[pruned node] (d2_4) at (1.0, 3.3) {}; \node[active node] (d2_5) at (2.9, 3.0) {};
    
    \node[active node, fill=green!40, label={[font=\footnotesize, text=green!60!black]below:\textbf{High Reward, Tail Mode}}] (leaf_1) at (-3.8, 1.2) {}; 
\node[active node, fill=green!40, label={[font=\footnotesize, text=green!60!black]below:\textbf{High Reward, Tail Mode}}] (leaf_8) at (3.5, 1.4) {};  
    \node[pruned node] (leaf_2) at (-2.6, 1.4) {}; \node[pruned node] (leaf_3) at (-1.4, 2.0) {}; \node[pruned node] (leaf_4) at (-0.6, 3.5) {}; \node[pruned node] (leaf_5) at (0.2, 3.7) {}; \node[pruned node] (leaf_6) at (0.8, 2.9) {}; \node[pruned node] (leaf_7) at (1.5, 1.2) {};

    \draw[active edge] (root) -- (d1_1); \draw[active edge] (d1_1) -- (d2_1); \draw[active edge] (d2_1) -- (leaf_1);
    \draw[active edge] (root) -- (d1_3); \draw[active edge] (d1_3) -- (d2_5); \draw[active edge] (d2_5) -- (leaf_8);

    \draw[pruned edge] (root) -- (d1_2); \draw[pruned edge] (d1_1) -- (d2_2); \draw[pruned edge] (d1_2) -- (d2_3); \draw[pruned edge] (d1_2) -- (d2_4); \draw[pruned edge] (d1_3) -- (d2_4);
    \draw[pruned edge] (d2_1) -- (leaf_2); \draw[pruned edge] (d2_2) -- (leaf_3); \draw[pruned edge] (d2_2) -- (leaf_4); \draw[pruned edge] (d2_3) -- (leaf_5); \draw[pruned edge] (d2_3) -- (leaf_6); \draw[pruned edge] (d2_4) -- (leaf_7);

    \newcommand{\prunemark}[2]{\draw[red, very thick] (#1-0.15, #2-0.15) -- (#1+0.15, #2+0.15); \draw[red, very thick] (#1-0.15, #2+0.15) -- (#1+0.15, #2-0.15);}
    \prunemark{-2.9}{2.3} \prunemark{-1.5}{4.0} \prunemark{0.2}{5.25} \prunemark{1.6}{3.9}

    \draw[backup arrow] (leaf_1) to[bend right=15] node[left, font=\normalsize\bfseries] {$\hat{v} \uparrow$} (d2_1);
    \draw[backup arrow] (d2_1) to[bend right=15] (d1_1);
    \draw[backup arrow] (d1_1) to[bend right=15] (root);
    \draw[backup arrow] (leaf_8) to[bend left=15] node[right, font=\normalsize\bfseries] {$\hat{v} \uparrow$} (d2_5);
    \draw[backup arrow] (d2_5) to[bend left=15] (d1_3);
    \draw[backup arrow] (d1_3) to[bend left=15] (root);

    \node[align=left, font=\scriptsize, color=red!80!black] (prunetxt) at (-0.1, 4.8) {\textbf{Pruned}\\(Low UCT Value)};
    \draw[->, red!80!black, shorten >=2pt] (prunetxt.south) -- (0.1, 5.25);

    \draw[->, thick, black!80] (-5.5, 0) -- (-5.5, 4.5) node[above, font=\small] {$p(x)$};
     \node[draw=green!50!black, fill=green!5, rounded corners, thick, font=\scriptsize\bfseries, align=center] (backuptxt) at (-4.2, 3.9) {Online Value Backup\\Soft Bellman Eq.};
   \draw[->, thick, green!60!black, shorten >=2pt] (backuptxt.east) -- (-3.2, 3.8);
     
    \node[anchor=north, align=left, font=\normalsize, text width=9.5cm] at (0, -0.8) {
        \textbf{Key Mechanisms:}\\
        $\bullet$ Terminal feedback $r(x_0)$ triggers value backup ($\hat{v} \uparrow$).\\
        $\bullet$ $\text{Select}^{\text{UCT}}$ amplifies target-aligned trajectories.\\
        $\bullet$ Aggressively prunes low-reward, off-target branches.
    };
\end{scope}

\end{tikzpicture}%
}
\caption{\textbf{\small{Overview of the Lévy Adaptive Tree Search (LATS) Framework.}} \textbf{\small{(Stage 1)}} \small{The generative model uses heavy-tailed Lévy noise rather than Gaussian noise, extending structural coverage to the data manifold's rare long-tail modes.} \textbf{\small{(Stage 2)}} \small{LATS constructs a dynamic search tree built upon the trained Lévy transition dynamics to systematically explore the latent space.} \textbf{\small{(Stage 3)}} \small{Guided by feedback, a soft Bellman value backup mechanism propagates reward signals upward, allowing the $\text{Select}^{\text{UCT}}$ score to aggressively prune low-value paths while continuously amplifying the discovery of high-utility targets without fine-tuning the L\'evy model.}}
\label{fig:framework}
\end{figure*}
At each online interaction step, we start at the root and repeatedly apply the selection rule $\small{\text{Select}^{\text{UCT}}(x_{t-1}) = \hat{v}(x_{t-1}) + \beta \sqrt{\frac{\log N(x_t)}{N(x_{t-1})}}}$ until we reach a terminal node, from which we sample. Here $\beta$ governs the exploration–exploitation trade-off. Crucially, \emph{because rollouts are driven by a Lévy diffusion process rather than a purely Gaussian one, its heavy-tailed jumps enable strong exploration}. After sampling, we observe proxy ground-truth feedback from a near-accurate classifier $f^{\phi}$, parameterized by $\phi$, which outputs a class-wise probability distribution for the generated sample. We compute the terminal reward $r(x_0)$, which is defined as:
\begin{equation}
    r(x_0) = \log \left( \max_{i \in \mathcal{C}_{tar}} \left( f^\phi(x_0)_i \right) \right)
\end{equation}
where $f^\phi(x_0)_i$ denotes the probability assigned to class $i$ by a pre-trained classifier ($f^{\phi}$), and $\mathcal{C}_{tar}$ denotes the target set. We emphasize that we leverage $f^{\phi}$ to emulate an expensive real‑world online feedback collection process, without access to the true underlying function values $f^{\phi}(x)$ for any input $x$ not in the query set. After evaluating the terminal node using the reward function $\hat{v}(x_0) = r(x_0)$, we then apply the soft Bellman backup following the Monte-Carlo estimate of Eqn.~\ref{eq:value} to recursively update each parent’s value from its children for $t = 0, \dots, T$. We also update the visit counts for all nodes along this path. Such an update after every observation enables the sampler to retain the statistics of all past outcomes/observations and drives increasingly informed trajectory exploration for future queries. 
We refer to our end-to-end sampling framework as L\'evy Adaptive Tree Sampling (LATS). We provide an illustrative figure (see~\ref{fig:framework}) that highlights the key building blocks of LATS. 
\section{Experiments and Results}\label{sec:exp}

\paragraph{Evaluation Metrics} 
Success rate alone cannot capture our goal of generating samples that reflect diverse target properties, since collapsing onto a single high-likelihood mode would trivially achieve 100 percent success while ignoring other target classes. To explicitly reward both diversity and quality of the generated samples across different target classes, we introduce the Quality-Aware Diverse Alignment Score (QADAS), defined as follows:
\vspace{-5pt}
\begin{equation}
    \small{\text{QADAS} = \sum_{\mathcal{C}_i \in \{\mathcal{C}_{tar}\}} \underbrace{\frac{P_{\mathcal{C}_i}}{N_i}}_{\text{Avg. Quality}} \cdot \underbrace{\log N_i}_{\text{Diversity}} ; \text{\>\>\>where } P_{\mathcal{C}_i} = \sum_{i=1}^{N_i} p_i}
\end{equation}
where $N_i$ denotes the total number of samples generated for class $i$ determined by a pretrained accurate classifier ($f^{\phi}$), and $\{\mathcal{C}_{tar}\}$ represents the set of all target classes. Higher values of $P_{\mathcal{C}_i}$ indicate greater confidence in the classifier’s predictions, signaling better quality in the generated samples. Moreover, the $\log N_i$ term imposes diminishing returns on repeatedly generating samples from the same class, thereby discouraging a lack of diversity. We further benchmark LATS against the baselines using standard FID and MMD. For both metrics, the reference distribution is constructed by drawing an equal number of samples from each target class. We benchmark LATS against the following \textbf{baselines}:
SMC~\citep{skreta2025feynman}: A state-of-the-art SMC-based inference-time correction approach for alignment; DTS~\citep{jain2025diffusion}: A tree-based diffusion sampler for search; FKS~\citep{singhal2025general}: An inference-time search approach based on Feynman-Kac Steering.
\vspace{-4pt}
\paragraph{Dataset Details and Evaluation Setting}
We evaluate LATS on four diverse benchmarks: MNIST-LT, CIFAR-LT, ImageNet-LT, and MICRO2D~\citep{robertson2024micro2d} (a large-scale heterogeneous microstructure informatics dataset), spanning domains from vision to materials science. We detail the specifics of each long-tail dataset in the Appendix. We evaluate LATS on each dataset using various diverse target sets and sampling budgets. 
\begin{table*}[t]
\centering
\caption{\textbf{\small{Quantitative Comparison with $\mathcal{B}=250$.}} \small{We report FID (lower is better), MMD (lower is better), and the QADAS diversity metric (higher is better) across two distinct target sets for each dataset. LATS consistently outperforms baselines in maintaining sample quality while uncovering diverse target modes.}}
\label{tab:main_results}
\begin{small} 
\setlength{\tabcolsep}{4pt} 
\resizebox{\textwidth}{!}{
\begin{tabular}{ll ccc ccc ccc}
\toprule
\multirow{2}{*}{\textbf{Target}} & \multirow{2}{*}{\textbf{Method}} & \multicolumn{3}{c}{\textbf{MNIST-LT}} & \multicolumn{3}{c}{\textbf{CIFAR-10-LT}} & \multicolumn{3}{c}{\textbf{ImageNet-LT}} \\
\cmidrule(lr){3-5} \cmidrule(lr){6-8} \cmidrule(lr){9-11}
& & FID $\downarrow$ & MMD $\downarrow$ & QADAS $\uparrow$ & FID $\downarrow$ & MMD $\downarrow$ & QADAS $\uparrow$ & FID $\downarrow$ & MMD $\downarrow$ & QADAS $\uparrow$ \\
\midrule
\multirow{4}{*}{Set 1} & SMC & 1.8137 & 13.748 & 11.8700 & 3.7810 & 3.8231 & 9.9300 & 2.1463 & 2.8476 & 12.9748 \\
& DTS    & 1.2532 & 9.4025 & 14.8631 & 0.9021 & 3.1071 & 13.0262 & 0.9444 & 1.7802 & 18.8476 \\
& FKS    & 1.0699 & 7.5111 & 14.0710 & 0.8150 & 2.4711 & 11.4592 & 1.4609 & 2.0285 & 17.4755 \\
\rowcolor{blue!15} &  \textbf{LATS}  & \textbf{0.7264} & \textbf{5.5532} & \textbf{18.7452} & \textbf{0.7223} & \textbf{1.7286} & \textbf{14.0967} & \textbf{0.7588} & \textbf{0.7692} & \textbf{20.1346} \\
\midrule
\multirow{4}{*}{Set 2} & SMC  & 4.0168 & 18.9856 & 6.7324 & 3.7970 & 6.8242 & 0.0000 & 3.4637 & 6.7032 & 5.2677 \\
& DTS     & 2.5921 & 13.8074 & 8.0134 & 1.5221 & 5.8091 & 5.1896 & 1.3731 & 3.2406 & 9.1326 \\
& FKS   & 2.6417 & 18.1248 & 8.2691 & 1.1242 & 2.8291 & 2.3727 & 2.2120 & 6.1204 & 5.0596 \\
\rowcolor{blue!15} & \textbf{LATS} & \textbf{1.1746} & \textbf{6.3109} & \textbf{9.9620} & \textbf{0.8776} & \textbf{1.8279} & \textbf{5.9670} & \textbf{1.0536} & \textbf{1.4400} & \textbf{12.1296} \\
\bottomrule
\end{tabular}%
}
\end{small}
\end{table*}

\paragraph{Results with Target Set that only includes Rare classes}
To study how LATS samples from
\begin{wraptable}{r}{0.30\textwidth} 
\vspace{-4pt}
\centering
\caption{\small{MICRO2D for Set S2.} }
\label{tab:material_results}
\small
\setlength{\tabcolsep}{2pt} 
\scriptsize 
\begin{tabular}{@{}l ccc@{}}
\toprule
\multirow{2}{*}{\textbf{Method}} & \multicolumn{3}{c}{\textbf{MICRO2D} ($\mathcal{B}=250$)} \\
\cmidrule(lr){2-4}
& FID $\downarrow$ & MMD $\downarrow$ & QADAS $\uparrow$ \\
\midrule
SMC  & 4.2154 & 29.8165 & 0.0000 \\
DTS  & 2.8699 & 17.7681 & 4.6196 \\
FKS  & 3.5499 & 27.7336 & 0.0000 \\
\midrule
\rowcolor{blue!15} \textbf{LATS} & \textbf{0.9464} & \textbf{4.2721} & \textbf{6.1041} \\
\bottomrule
\end{tabular}
\vspace{-6pt} 
\end{wraptable}
long-tail target classes, we construct target sets consisting solely of tail classes from each dataset. Concretely, we use digits \{\textcolor{blue}{7, 9}\} on MNIST, \{\textcolor{blue}{horse, ship, truck}\} on CIFAR10, \{\textcolor{blue}{shoe, green apple, bear}\} on ImageNet, and \{\textcolor{blue}{VoronoiLarge}\} on MICRO2D, which are among the least represented classes during training and thus faithfully characterize the long-tail regime on these datasets. We denote these sets as $S2$. We report the performance of LATS against the baselines on the vision tasks (Table~\ref{tab:main_results}) and the material dataset (Table~\ref{tab:material_results}).
Across all settings, LATS consistently surpasses the baselines, preserving sample quality while revealing a broader range of high-utility target modes. Interestingly, we observe (as depicted in Figure~\ref{fig:visual_comparisons_all_1}) that DTS in this setting often produces low-quality images for long-tail target classes. During tree rollouts, trajectories rarely visit these tail regions. On the other hand, the search prunes high-likelihood head-class modes as low value and instead forces the diffusion model to generate samples outside its well-supported modes on the data manifold, yielding degraded, off-manifold images. 
\vspace{-4pt}
\paragraph{Results with Target Set that is a Mixture of Frequent and Rare classes}
We benchmark LATS in a mixed head–tail regime where the target set spans both frequent and rare classes. Concretely, we take digits \{1, 3, 5, 7, 9\} on MNIST, \{airplane, bird, deer, frog, ship\} on CIFAR10, and \{dog, car, shark, frog, green apple\} on ImageNet, which together capture some of the most and least represented training classes on each dataset. We refer to this mixture as set $S1$ and report results in Table~\ref{tab:main_results}; visualizations are in Fig.~\ref{fig:visual_comparisons_all_1}. In this setting, LATS consistently surpasses all baselines across evaluation metrics, while the baselines exhibit mode collapse, generating almost exclusively from the high-utility, high-likelihood head class, in line with Theorem~\ref{th:it-collapse}. Overall, these results show that LATS preserves sample quality while uncovering diverse target modes, including those in the long tail. SMC underperforms because they depend on an online-trained reward model whose early-stage bias degrades both image quality and alignment. In contrast, DTS fails to achieve high QADAS since its Gaussian-based diffusion guidance constrains the denoising trajectories, preventing effective exploration of long-tail modes and suppressing diversity.
\paragraph{Effect of L\'evy Process on Exploration}
Comparison with DTS (Tables~\ref{tab:main_results} and Fig.~\ref{fig:visual_comparisons_all_1}) validates the Lévy process's critical role in LATS: it unlocks long-tail discovery, reaching high-utility modes that Gaussian dynamics fail to access.
These high-utility tail modes are then amplified by the tree sampler’s online value backup, leading to substantially more frequent sampling from high-utility regions in the long-tail.
\begin{figure*}[h]
    \centering
    \begin{subfigure}{0.31\textwidth}
        \centering
        \includegraphics[height=1.6cm,width=\linewidth,keepaspectratio,interpolate=true]{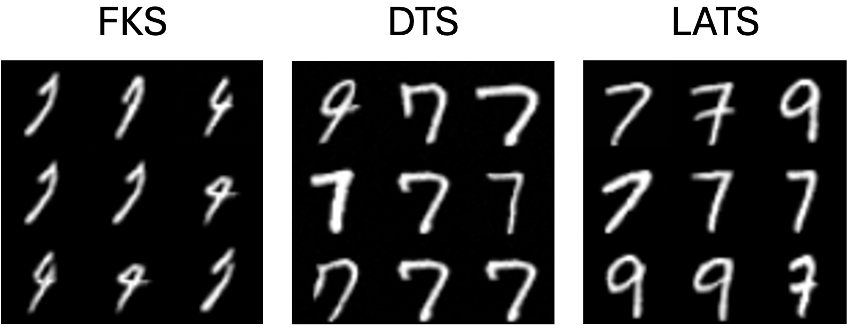}
        \caption{\{\textcolor{blue}{7}, \textcolor{blue}{9}\}}
        \label{fig:mnist_s2}
    \end{subfigure}
    \hfill
    \begin{subfigure}{0.33\textwidth}
        \centering
        \includegraphics[height=1.6cm,width=\linewidth,keepaspectratio,interpolate=true]{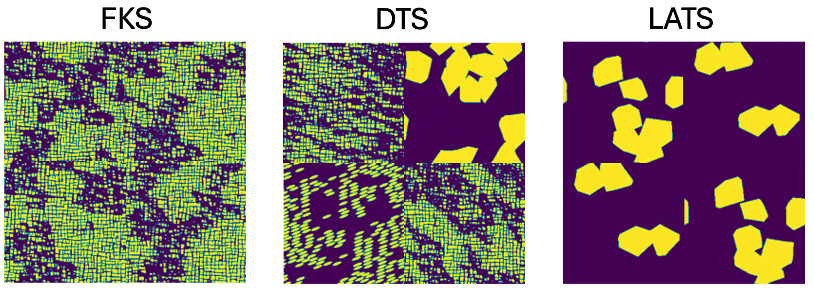}
        \caption{\{\textcolor{blue}{{VoronoiLarge}}\}}
        \label{fig:micro2d_s2}
    \end{subfigure}
    \hfill
    \begin{subfigure}{0.34\textwidth}
        \centering
        \includegraphics[height=1.6cm,width=\linewidth,keepaspectratio,interpolate=true]{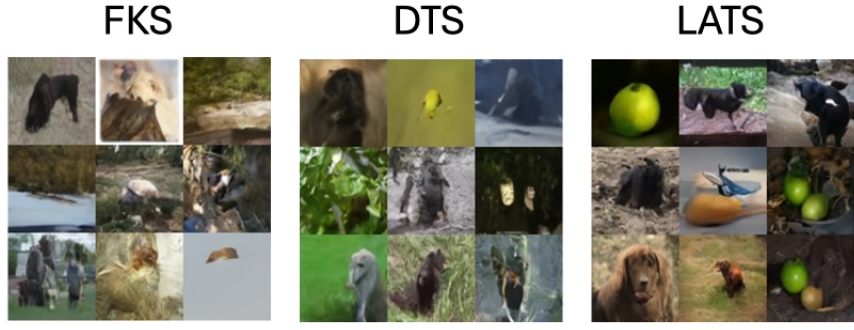}
        \caption{{\textcolor{blue}{Shoe}, \textcolor{blue}{Green Apple}, \textcolor{blue}{bear}}}
        \label{fig:imagenet_s2}
    \end{subfigure}

    \vspace{1em} 

    \begin{subfigure}{0.29\textwidth}
        \centering
        \includegraphics[height=1.6cm,width=\linewidth,keepaspectratio,interpolate=true]{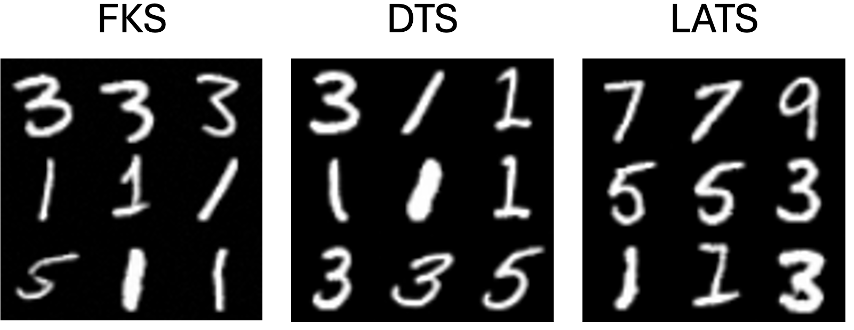}
        \caption{\{1, 3, 5, \textcolor{blue}{7}, \textcolor{blue}{9}\}}
        \label{fig:mnist_s1}
    \end{subfigure}
    \hfill
    \begin{subfigure}{0.35\textwidth}
        \centering
        \includegraphics[height=1.6cm,width=\linewidth,keepaspectratio,interpolate=true]{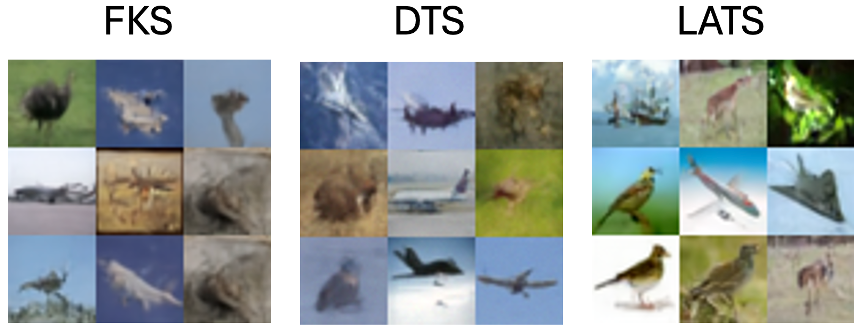}
        \caption{\{Plane, Bird, Deer, Frog, \textcolor{blue}{Ship}\}}
        \label{fig:cifar_s1}
    \end{subfigure}
    \hfill
    \begin{subfigure}{0.34\textwidth}
        \centering
        \includegraphics[height=1.6cm,width=\linewidth,keepaspectratio,interpolate=true]{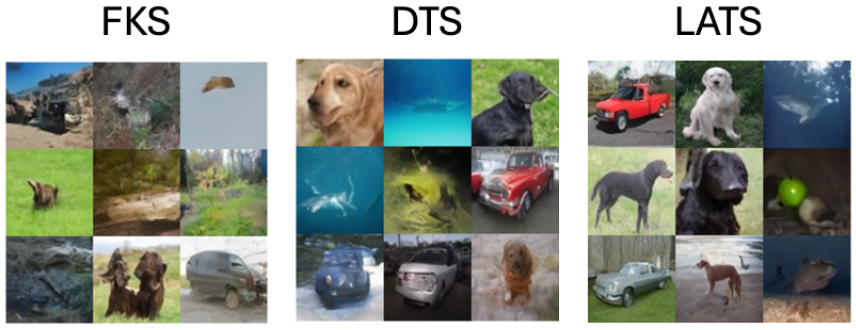}
        \caption{Dog, Car, Shark, Frog, \textcolor{blue}{Apple}}
        \label{fig:imagenet_s1}
    \end{subfigure}

    \caption{\textbf{\small{Qualitative comparison of generation diversity and quality across target sets and datasets.}} While DTS and FKS tend to focus on dominant modes, LATS uncovers diverse, high-utility samples by leveraging Lévy-driven exploration and tree-based online value backup, aiding history-aware exploitation while maintaining sample quality. {long-tail high-utility classes are denoted as \textcolor{blue}{blue}.}}
    \label{fig:visual_comparisons_all_1}
    \vspace{-14pt}
\end{figure*}
\vspace{-2pt}
\vspace{-2pt}
\paragraph{Performance Comparison at different Sampling Budgets}
Figure~\ref{fig:main_performance_curves} reports LATS and baseline performance across sampling budgets. Across datasets and metrics, LATS converges fastest in both quality and target discovery while sustaining markedly higher sample diversity, showcasing its efficacy in diverse rare-target discovery under a strict sampling budget. 
\begin{figure*}[!h]
    \centering
    \begin{subfigure}{0.32\linewidth}
        \centering
        \includegraphics[width=\linewidth]{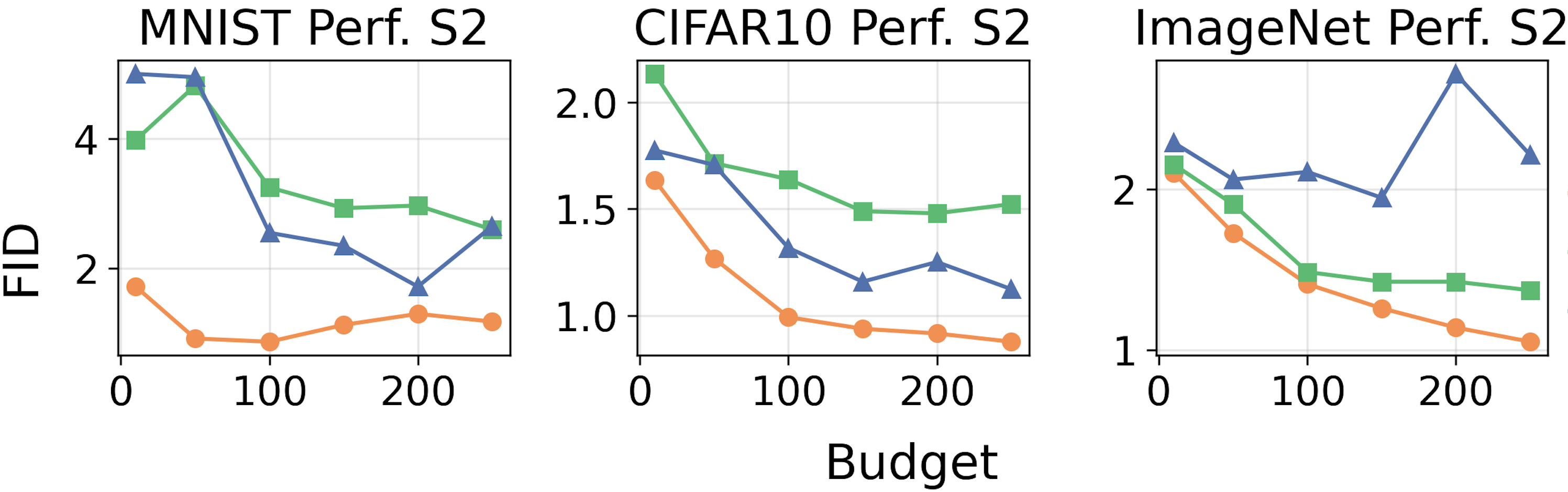}
        \caption{FID $\downarrow$}
        \label{fig:fid_results}
    \end{subfigure}
    \hfill
    \begin{subfigure}{0.32\linewidth}
        \centering
        \includegraphics[width=\linewidth]{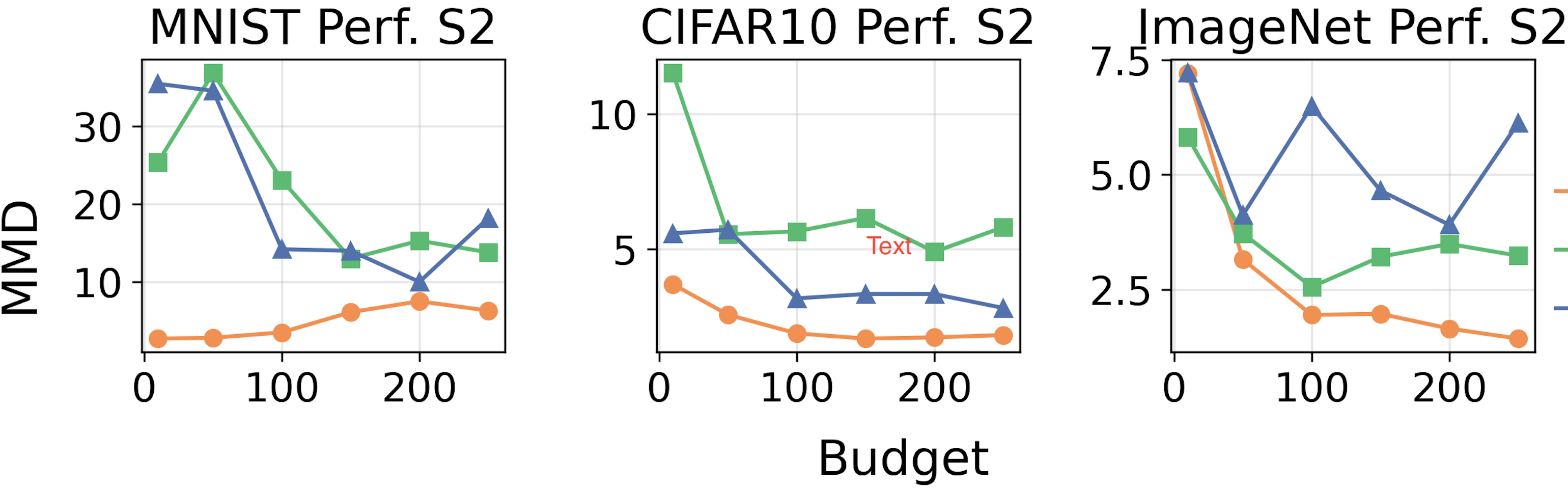}
        \caption{MMD $\downarrow$}
        \label{fig:mmd_results}
    \end{subfigure}
    \hfill
    \begin{subfigure}{0.34\linewidth}
        \centering
        \includegraphics[width=\linewidth]{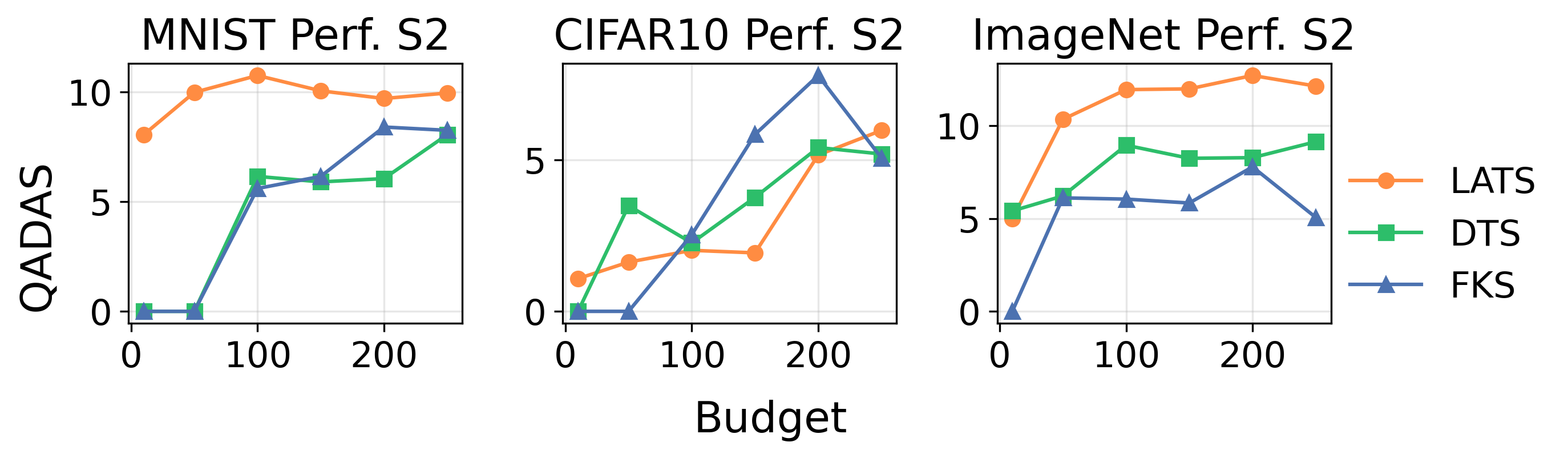}
        \caption{QADAS $\uparrow$}
        \label{fig:qadas_results}
    \end{subfigure}
    \caption{\textbf{Quantitative Comparisons across Sampling Budgets.} Comparison of LATS against baselines across MNIST-LT, CIFAR-10-LT, and ImageNet-LT for target Set S2. LATS achieves the fastest convergence in quality and target discovery while maintaining significantly higher sample diversity as the sampling budget increases.}
    \label{fig:main_performance_curves}
\end{figure*}
\vspace{-4pt}

\vspace{-10pt}
\paragraph{Conclusions and Future Work}
LATS facilitates the feedback-efficient discovery of diverse targets, specifically capturing those within the long-tail modes. Crucially, our ablation studies establish that L\'evy process-guided trajectory rollouts for efficient exploration, paired with the tree sampler's backward online value propagation for exploitation, are indispensable. Together, these components synergistically uncover high-utility rare modes without compromising sample quality. 
While the L\'evy diffusion process is integral to exploration in LATS, training it relies on a well-defined long-tail dataset. Because this data requirement can be prohibitive in resource-constrained domains like drug discovery, adapting the framework for data-scarce environments—where pre-training a robust L\'evy model is infeasible—represents a critical direction for future work.

\bibliographystyle{unsrt}
\bibliography{nips}


\appendix

\appendix
\onecolumn
\section*{LATS: L\'evy Adaptive Tree Sampling for Alignment with Online Feedback (Supplementary Material)}

\section{PROOF OF PROPOSITION 1} 
\begin{proof}
We denote the number of all training images as $N$, and the number of training images for class $c^{i} \in \mathcal{C}$ as $N^{i}$. Following~\citep{ho2020denoising}, the training objective of DDPM can be rewritten as:

\begin{align}
&\mathbb{E}_{q}\Bigg[
    \sum_{t\ge1}^{T}
    D_{\mathrm{KL}}\bigl(
        q(x_{t-1}\mid x_t,x_0,c)
        \,\big\|\,
        p_{\theta}(x_{t-1}\mid x_t,c)
    \bigr)
\Bigg] \nonumber \\
&=
\sum_{t\ge1}^{T}
\mathbb{E}_{q}\Big[
    D_{\mathrm{KL}}\bigl(
        q(x_{t-1}\mid x_t,x_0,c)
        \,\big\|\,
        p_{\theta}(x_{t-1}\mid x_t,c)
    \bigr)
\Big] \nonumber \\
&=
\sum_{t\ge1}^{T}
\Bigg[
    \frac{1}{N}
    \sum_{(x_0,c)}
    D_{\mathrm{KL}}\bigl(
        q(x_{t-1}\mid x_t,x_0,c)
        \,\big\|\,
        p_{\theta}(x_{t-1}\mid x_t,c)
    \bigr)
\Bigg] \nonumber \\
&=
\sum_{t\ge1}^{T}
\Bigg[
    \frac{1}{N}
    \sum_{c^i\in\mathcal{C}}
    \sum_{(x_0,c^i)}
    D_{\mathrm{KL}}\bigl(
        q(x_{t-1}\mid x_t,x_0,c^i)
        \,\big\|\,
        p_{\theta}(x_{t-1}\mid x_t,c^i)
    \bigr)
\Bigg] \nonumber \\
&=
\sum_{t\ge1}^{T}
\Bigg[
    \frac{1}{N}
    \sum_{c^i\in\mathcal{C}}
    N^{i}\,
    \mathbb{E}_{q,c^i}\Big(
        D_{\mathrm{KL}}\bigl(
            q(x_{t-1}\mid x_t,x_0,c^i)
            \,\big\|\,
            p_{\theta}(x_{t-1}\mid x_t,c^i)
        \bigr)
    \Big)
\Bigg] \nonumber \\
&=
\sum_{t\ge1}^{T}
\Bigg[
    \sum_{c^i\in\mathcal{C}}
    \frac{N^{i}}{N}\,
    \mathbb{E}_{q,c^i}\Big(
        D_{\mathrm{KL}}\bigl(
            q(x_{t-1}\mid x_t,x_0,c^i)
            \,\big\|\,
            p_{\theta}(x_{t-1}\mid x_t,c^i)
        \bigr)
    \Big)
\Bigg] \nonumber \\
&=
\sum_{c^i\in\mathcal{C}}
\frac{N^{i}}{N}\,
\mathbb{E}_{q,c^i}\Bigg[
    \sum_{t\ge1}^{T}
    D_{\mathrm{KL}}\bigl(
        q(x_{t-1}\mid x_t,x_0,c^i)
        \,\big\|\,
        p_{\theta}(x_{t-1}\mid x_t,c^i)
    \bigr)
\Bigg] \nonumber \\
&=
\sum_{c^i\in\mathcal{C}}
\underbrace{w_i}_{\text{Learning Bias}}\,
\mathbb{E}_{q,c^i}\Bigg[
    \sum_{t\ge1}^{T}
    D_{\mathrm{KL}}\bigl(
        q(x_{t-1}\mid x_t,x_0,c^i)
        \,\big\|\,
        p_{\theta}(x_{t-1}\mid x_t,c^i)
    \bigr)
\Bigg],
\label{eq:ddpm_weighted_decomp}
\end{align}

where $w_i = \frac{N^{i}}{N}$ is the ratio of class $c^{i}$ images in the training set.

\end{proof}

\section{PROOF OF THEOREM 2} \label{app:th1}
\begin{proof}
Assume that, after the first epoch of training on the long-tail dataset $\mathcal{D}_{lt}$, the pre-trained generative model induces the data distribution $\mathcal{P}(x_0)$. Given the standard DDPM denoising loss function:
\[
\mathcal{L} (\theta) = \mathbb{E}_{x_0 \sim \mathcal{P}(x_0)} \underbrace{\left[ \sum_{t=1}^{T} \frac{1- \alpha_t}{\alpha_t (1 - \bar{\alpha}_{t-1})} \left\| \epsilon_0 - \epsilon_\theta(x_t, t) \right\|^2 \right]}_{\textit{Let's denote it as $L(x_0)$}}  \tag{1}
\]
Here, $L(x_0)$ is a loss function defined for sample $x_0$. Now, if we fine-tune the parameters of the diffusion model (i.e., $\theta$) with the objective defined in 1, then according to~\cite{peng2019advantage} the induced data distribution $\mathcal{P}^{induced}(x_0)$ can be defined as:
\[
\mathcal{P}^{induced}(x_0) \propto \mathcal{P}(x_0) \exp (-\gamma L(x_0))    \tag{2}
\]
with $\gamma > 0$ being a positive constant. Now, consider a weighted probability distribution $\mathcal{P}^{weighted}(x_0) = w(x_0)\mathcal{P}(x_0)$. In that case, we can represent the induced distribution $\mathcal{P}^{induced}(x_0)$ in terms of $\mathcal{P}(x_0)$, as follows:
\[
\mathcal{P}^{induced}(x_0) \propto w(x_0)\mathcal{P}(x_0) \exp (-\gamma L(x_0))    \tag{3}
\]
We derive 3, by observing that 
\[
\mathcal{L}(\theta) = \mathbb{E}_{x_0 \sim \mathcal{P}^{weighted}(x_0)}[L(x_0)]\]
This, in turn, implies that 
\[\mathcal{P}^{induced}(x_0) \propto \underbrace{\mathcal{P}^{weighted}(x_0)}_{\textit{$w(x_0)\mathcal{P}(x_0)$}} \exp (-\gamma L(x_0))\]

We can re-write the weighted induced objective function $\mathcal{L}^{induced}_{}(\theta)$ as follows:
\[
 = \mathbb{E}_{t \sim U[0,1], x_0 \sim \mathcal{P}(x_0), x_t \sim p_t(x_t|x_0)} \left[ w(x_0) \cdot \underbrace{\frac{1- \alpha_t}{\alpha_t (1 - \bar{\alpha}_{t-1})}}_{\textit{= c > 0}} \| \epsilon_{\theta}(x_t, t) - \epsilon_0) \|^2 \right]  \tag{4}
\]
Where
\[
w(x_0) = r(x_0) 
\]
Note that $\gamma > 0$. Assuming there is at least a single element in $\mathcal{S}_{maj}$, implies $w(x_0) > 0$.
\[
\mathcal{L}^{induced}(\theta) = \int_0^1  \int_{\mathcal{X}} w(x_0) \mathcal{P}(x_0) \int_{\mathcal{X}} \| \epsilon_{\theta}(x_t; t) - \epsilon_0) \|^2 \mathcal{P}_t(x_t | x_0) \, dx_t \, dx_0 \, dt \tag{5}
\]
Then, we can define the reweighted distribution \( p^{\text{induced}}(x_0) \) as:
\[
p^{\text{induced}}(x_0) = \frac{w(x_0) \mathcal{P}(x_0)}{Z}, \quad \text{where} \quad Z = \int w(x_0) \mathcal{P}(x_0) \, dx_0 \tag{6}
\]
Since \( w(x_0) \geq 0 \) and \( Z < \infty \), \( p^{\text{induced}}(x_0) \) is a valid pdf over \( \mathcal{X} \).

Substituting $p^{\text{induced}}(x_0)$ (from 6) into the loss function in 5, we have:
\[
\mathcal{L}^{induced}(\theta) = Z \int_0^1  \int_{\mathcal{X}} p^{\text{induced}}(x_0) \int_{\mathcal{X}} \| \epsilon_{\theta}(x_t; t) - \epsilon_0) \|^2 \mathcal{P}_t(x_t | x_0) \, dx_t \, dx_0 \, dt \tag{7}
\]

We can rewrite the above expression as:
\[
 = Z \cdot \mathbb{E}_{t \sim U[0,1], x_0 \sim \mathcal{P}^{\text{induced}}(x_0), x_t \sim p_t(x_t|x_0)} \left[ \| \epsilon_{\theta}(x_t, t) - \epsilon_0) \|^2 \right] \>\>\>\> \tag{8} 
\]
Therefore, the gradient of the above expression is:
\[
\mathcal{L}^{induced}(\theta) = (Z)\> \mathbb{E}_{t \sim U[0,1], x_0 \sim \mathcal{P}^{\text{induced}}(x_0), x_t \sim p_t(x_t|x_0)} \left[ \nabla_{\theta} \| \epsilon_{\theta}(x_t, t) - \epsilon_0) \|^2\right]  \tag{9}
\]
 Note that the normalizing factor Z does not depend on the optimization variable $\theta$. Hence, it does not affect the optimization process. Therefore, minimizing $\mathcal{L}^{induced}(\theta)$ is equivalent to minimizing the expected loss under the distribution $\mathcal{P}^{\text{induced}}(x_0)$.
 
 Following a similar result as in Equation 2, we can express the learned data distribution by optimizing Equation 9, denoted as $\mathcal{P}^1_{\theta}(x_0)$ as follows:
 \[
 \mathcal{P}^{1}_{\theta}(x_0) \propto \mathcal{P}^{\text{induced}}(x_0) \exp (-\gamma \mathcal{L}^{induced}(x_0;\theta) ) \tag{10}
 \]
 Utilizing the relation from Equation 6, we can write:
 \[
 \mathcal{P}^{1}_{\theta}(x_0) \propto w(x_0) \mathcal{P}(x_0) \exp (-\gamma \mathcal{L}^{induced}(x_0;\theta) )   \tag{11}
 \]  
 Assuming the $\mathcal{L}^{induced}(x_0;\theta)$ 
 loss converges to $0$ after the fine-tuning step, we can write the above expression as:
 \[
 \mathcal{P}^{1}_{\theta}(x_0) \propto w(x_0) \mathcal{P}(x_0)   \tag{12}
 \] 
 By normalizing with the normalization constant Z, we can write:
 \[
 \mathcal{P}^{1}_{\theta}(x_0) = \frac{w(x_0) \mathcal{P}(x_0)}{Z}   \tag{13}
 \]    

 We can now iteratively repeat these steps, and by induction, express the learned data distribution after $H$ update steps (i.e., $H$ online interaction steps) as follows:
 \[
 \mathcal{P}^{H}_{\theta}(x_0) = \frac{w(x_0)^{H} \mathcal{P}(x_0)}{Z^H} , \text{where,} Z^H = \int_{\mathcal{X}}w(x_0)^H\mathcal{P}(x_0) dx_0   \tag{14}
 \]   

Hence, following the notation of Theorem 2, the probability density after $N$ epochs is $g^N_{\theta^{\mathcal{D}_{lt}}_0}(x) \propto w(x)^N g^0_{\theta^{\mathcal{D}_{lt}}_0}(x)$. We examine the ratio of the total probability mass assigned to the majority class versus the tail class after $N$ epochs:$$\text{Ratio}^N = \frac{P^N(\mathcal{C}_{maj})}{P^N(\mathcal{C}_{tail})} = \frac{\int_{\mathcal{C}_{maj}} g^N_{\theta^{\mathcal{D}_{lt}}_0}(x) dx}{\int_{\mathcal{C}_{tail}} g^N_{\theta^{\mathcal{D}_{lt}}_0}(x) dx}$$
Substituting the above result from Equation 14 ($q^N \propto w^N q^0$):$$\text{Ratio}^N = \frac{\int_{\mathcal{C}_{maj}} w(x)^N g^0_{\theta^{\mathcal{D}_{lt}}_0}(x) dx}{\int_{\mathcal{C}_{tail}} w(x)^N g^0_{\theta^{\mathcal{D}_{lt}}_0}(x) dx}$$
Let $\bar{w}_{maj}$ and $\bar{w}_{tail}$ be the representative (or average) weights for the majority and tail classes, respectively. Based on our assumption that the reward follows the data distribution generated by $g^0_{\theta^{\mathcal{D}_{lt}}_0}(x)$. Let $\bar{w}_{maj} = \bar{w}_{tail} + \epsilon$ for some $\epsilon > 0$.$$\text{Ratio}^N \approx \frac{\bar{w}_{maj}^N \cdot P_0(\mathcal{C}_{maj})}{\bar{w}_{tail}^N \cdot P_0(\mathcal{C}_{tail})}$$
We can rewrite the ratio as:$$\text{Ratio}^N = \left( \frac{\bar{w}_{maj}}{\bar{w}_{tail}} \right)^N \cdot \underbrace{\frac{P_0(\mathcal{C}_{maj})}{P_0(\mathcal{C}_{tail})}}_{\text{Initial Imbalance}}$$Since $\bar{w}_{maj} > \bar{w}_{tail}$, the ratio $\frac{\bar{w}_{maj}}{\bar{w}_{tail}} > 1$.Therefore, as $N \to \infty$:$$\left( \frac{\bar{w}_{maj}}{\bar{w}_{tail}} \right)^N \to \infty$$

The probability mass ratio diverges to infinity. This implies $P^N(\mathcal{C}_{tail}) \to 0$ relative to $P^N(\mathcal{C}_{maj})$.
Thus, the model exhibits mode collapse towards the highest sample class, exacerbating the initial long-tail imbalance exponentially with every fine-tuning epoch.
 
 This completes the proof.
\end{proof}

\section{PROOF OF PROPOSITION 2} \label{app:th2}

\subsection{The Non-Local Score Gradient}

\begin{proof}
The reverse diffusion drift is $\mathbf{s}_t(x) = \nabla \log p_t(x)$. In the Gaussian case, $\nabla \log p_t(x)$ is a local operator. In the Lévy model, the score is related to the fractional Laplacian:
\begin{equation}
(-\Delta)^{\alpha/2} p(x) = \int_{\mathbb{R}^d} \frac{p(x) - p(y)}{\|x-y\|^{d+\alpha}} dy
\end{equation}
Because the denominator $\|x-y\|^{d+\alpha}$ is a polynomial, the gradient signal from mode $B$ reaches mode $A$ with significant strength. This allows the Langevin-like MCMC sampling in the reverse process to remain "aware" of disjoint modes, effectively mitigating mode collapse.

To show how the L\'evy model mitigates mode collapse compared to DDPM, we must analyze the "score signal" (the gradient of the log-density) that the model receives during training. Mode collapse occurs when the gradient signal from a minority (tail) mode is exponentially suppressed by the majority mode, causing the neural network to ignore the minority class.

\textbf{Score Signal Persistence in Heavy-Tailed Diffusion} The Setup: Assume an Imbalanced Two-Mode Distribution. Consider a data distribution $p_{data}$ with a majority mode at $0$ and a rare minority mode at $\mu \gg 0$, separated by a large distance. $$p_{data}(x) = (1 - \pi) \delta_0(x) + \pi \delta_\mu(x)$$ where $\pi \in (0, 1)$ is the mixing proportion. In a long-tail scenario, $\pi \to 0$. In diffusion models, we observe the score of the noised distribution $p_t(x) = (p_{data} * K_t)(x)$, where $K_t$ is the noise kernel at time $t$. DDPM (Gaussian): $K_t^{G}(x) = \frac{1}{\sqrt{2\pi \sigma_t^2}} \exp\left(-\frac{x^2}{2\sigma_t^2}\right)$ DLPM ($\alpha$-Stable): $K_t^{L}(x) \sim \frac{C_{\alpha, \sigma_t}}{|x|^{1+\alpha}}$ (as $|x| \to \infty$).  We evaluate the score function $\nabla_x \log p_t(x)$ near the minority mode $x \approx \mu$. For DDPM, the density is: $$p_t^G(x) = (1-\pi) \mathcal{N}(x; 0, \sigma_t^2) + \pi \mathcal{N}(x; \mu, \sigma_t^2)$$The score function $s^G(x, t) = \nabla_x \log p_t^G(x)$ is a weighted sum:$$s^G(x, t) = \frac{(1-\pi) \nabla \mathcal{N}(x; 0, \sigma_t^2) + \pi \nabla \mathcal{N}(x; \mu, \sigma_t^2)}{(1-\pi) \mathcal{N}(x; 0, \sigma_t^2) + \pi \mathcal{N}(x; \mu, \sigma_t^2)}$$At the location of the minority mode $x = \mu$, the influence of the majority mode relative to the minority mode is determined by the ratio:$$R_G = \frac{\text{Mass from Majority}}{\text{Mass from Minority}} = \frac{(1-\pi) \exp(-\mu^2 / 2\sigma_t^2)}{\pi \exp(0)} = \frac{1-\pi}{\pi} e^{-\frac{\mu^2}{2\sigma_t^2}}$$As $\mu$ (distance) increases, $R_G \to 0$ exponentially. This means for any large $\mu$, the density $p_t^G(\mu)$ is effectively zero unless $\sigma_t$ is very large. Consequently, the gradient signal used to train the model is zeroed out by the exponential decay of the Gaussian tail. The model "forgets" that the mode exists. Now consider the L\'evy kernel $K_t^L(x)$. For a point $x$ near the minority mode $\mu$:$$p_t^L(x) = (1-\pi) K_t^L(x) + \pi K_t^L(x-\mu)$$Using the polynomial tail property $K_t^L(x) \approx C |x|^{-(1+\alpha)}$, the ratio of mass at $x = \mu$ is:$$R_L = \frac{(1-\pi) K_t^L(\mu)}{\pi K_t^L(0)} \approx \frac{(1-\pi) C \mu^{-(1+\alpha)}}{\pi K_t^L(0)}$$Unlike the Gaussian case, $R_L$ decays polynomially ($1/\mu^{1+\alpha}$). In the training objective (Equation 6 in the paper), the model minimizes:$$\mathcal{L} = \int p_t(x) \| s_\theta(x,t) - \nabla \log p_t(x) \|^2 dx$$In DDPM: The region around $\mu$ has a weight $p_t^G(\mu) \propto e^{-\mu^2}$. Because this weight is exponentially small, the total loss $\mathcal{L}$ is minimized even if the model $s_\theta$ completely fails to predict the score at the minority mode.

In L\'evy: The weight $p_t^L(\mu) \propto \mu^{-(1+\alpha)}$ remains significant even for large distances. The heavy tails of the Lévy process ensure that the "influence" of the majority mode reaches the minority mode with polynomial strength. The optimizer cannot ignore the minority mode because the heavy-tail overlap contributes a non-negligible amount to the total loss. The proof shows that L\'evy mitigates mode collapse because the $\alpha$-stable noise kernel maintains an algebraic (polynomial) connection between distant modes, whereas Gaussian noise creates an exponential disconnection. This ensures that the gradient signal for rare tail classes remains "visible" to the neural network throughout the training process.
\end{proof}

\section{Additional Discussion on $\alpha$-Stable Distribution and comparison with Gaussian}

We define mode collapse as the inability of a generative model to assign mass to a subset of the target distribution $p_{data}$ when that subset is separated by a low-density region.

\paragraph{Diffusion Transition Kernel} Let $X_t$ be a stochastic process. The probability of $X_t$ moving from mode $A$ to mode $B$ at distance $L$ is governed by the transition kernel $K_t(x, y)$.

\paragraph{1. Gaussian Case ($\alpha = 2$)}
In standard Denoising Diffusion Probabilistic Models (DDPM), the kernel is:
\begin{equation}
K_t^{Gauss}(x, y) = \frac{1}{(4\pi t)^{d/2}} \exp\left( -\frac{\|x-y\|^2}{4t} \right)
\end{equation}
For any $x \in A$ and $y \in B$, if $\|x-y\| \ge L$, the probability density satisfies:
\begin{equation}
\lim_{L \to \infty} \frac{\log K_t^{Gauss}(L)}{-L^2} = \text{const} > 0
\end{equation}
This indicates an \textbf{exponential suppression} of long-range dependencies. If $L$ is large, the reverse process cannot "jump" between modes, leading to collapse into the nearest local optimum.

\paragraph{2. DLM Case ($\alpha < 2$)}
DLM utilizes the $\alpha$-stable Lévy process. The transition kernel $K_t^{L\acute{e}vy}(x, y)$ for large $\|x-y\|$ follows:
\begin{equation}
K_t^{L\acute{e}vy}(x, y) \approx t \cdot C_{d,\alpha} \|x-y\|^{-(d+\alpha)}
\end{equation}
where $C_{d,\alpha} = \frac{2^\alpha \Gamma((d+\alpha)/2)}{\pi^{d/2} |\Gamma(-\alpha/2)|}$.

\paragraph{ Comparison of Support Coverage}
Let $P(X_t \in B | X_0 \in A)$ be the probability of reaching mode $B$.
\begin{itemize}
    \item \textbf{Gaussian:} $P(B|A) \sim \int_L^\infty e^{-r^2} dr \approx \text{erfc}(L) \approx \frac{e^{-L^2}}{L\sqrt{\pi}}$
    \item \textbf{DLM:} $P(B|A) \sim \int_L^\infty r^{-(d+\alpha)} r^{d-1} dr \sim L^{-\alpha}$
\end{itemize}

\begin{proposition}
For any $\epsilon > 0$ and any finite distance $L$, there exists an $\alpha \in (0, 2)$ such that $P_{DLM}(B|A) > P_{Gauss}(B|A)$. Specifically, as $L \to \infty$:
\begin{equation}
\frac{P_{DLM}(B|A)}{P_{DDPM}(B|A)} \propto \frac{L^{-\alpha}}{e^{-L^2}} \to \infty
\end{equation}
\end{proposition}
This relation reveals that DLM achieves a substantially higher probability of reaching the long tail compared to standard Gaussian DDPM.

\section{Impact of UCT-based Exploration on Performance}
In this section, we investigate how UCT-style exploration drives efficient target discovery by varying the exploration strength $\beta$ and reporting performance across settings in Table~\ref{tab:beta_sensitivity}. When exploration is disabled ($\beta=0$), performance drops significantly across domains and evaluation metrics, underscoring that exploitation alone is insufficient for effective and diverse target discovery. Introducing even modest exploration ($\beta=0.1$) yields a clear recovery, while the strongest overall results are achieved at $\beta=1$, highlighting that exploration is crucial for efficient target discovery.

\begin{table}[h]
\centering
\caption{\textbf{Analysis of Different Exploration Factor $\beta$.} We evaluate performance under different Exploration factor $\beta$ with a sampling budget of $250$. For this study, the target sets on each dataset include a mix of frequent and rare classes (i.e., Set $S1$).} 
\label{tab:beta_sensitivity}
\setlength{\tabcolsep}{2.5pt} 
\begin{tabular}{@{}l ccc ccc@{}}
\toprule
\multirow{2}{*}{\textbf{Method}} & \multicolumn{3}{c}{\textbf{CIFAR-10-LT}} & \multicolumn{3}{c}{\textbf{ImageNet-LT}} \\
\cmidrule(lr){2-4} \cmidrule(lr){5-7} 
& FID $\downarrow$ & MMD $\downarrow$ & QADAS $\uparrow$ & FID $\downarrow$ & MMD $\downarrow$ & QADAS $\uparrow$ \\
\midrule
\textbf{LATS} w $\beta$ = 0  & 0.8657 & 2.2700 & 12.6108 & 0.8976 & 0.8637 & 18.7534 \\
\textbf{LATS} w $\beta$ = 0.1   & 0.8264 & 2.0361 & 13.3479 & 0.7846 & 0.8211 & 18.9304 \\
\midrule
\textbf{LATS} w $\beta$ = 1 & \textbf{0.7223} & \textbf{1.7286} & \textbf{14.0967} & \textbf{0.7588} & \textbf{0.7692} & \textbf{20.1346} \\
\bottomrule
\end{tabular}
\end{table}

\paragraph{Impact of Explicit UCT-Style Tree Exploration}
We probe the role of UCT-style node selection by ablating the exploration term in the node selection mechanism during sampling and comparing against full LATS. As shown in Row 1 of Table~\ref{tab:preference_alignment_results}, removing this term yields markedly lower $\text{QADAS}$, indicating that explicit UCT-style exploration is crucial for efficient mode discovery and higher sample diversity.
\vspace{-6pt}
\begin{table}[h]
\vspace{-16pt}
\centering
\caption{\textbf{Analysis of Different Components of LATS.} We evaluate performance under a sampling budget of 1000. For this study, the target sets on each dataset include a mix of frequent and rare classes (i.e., Set $S1$).} 
\label{tab:preference_alignment_results}
\scriptsize 
\setlength{\tabcolsep}{2.5pt} 
\begin{tabular}{@{}l ccc ccc@{}}
\toprule
\multirow{2}{*}{\textbf{Method}} & \multicolumn{3}{c}{\textbf{CIFAR-10-LT}} & \multicolumn{3}{c}{\textbf{ImageNet-LT}} \\
\cmidrule(lr){2-4} \cmidrule(lr){5-7} 
& FID $\downarrow$ & MMD $\downarrow$ & QADAS $\uparrow$ & FID $\downarrow$ & MMD $\downarrow$ & QADAS $\uparrow$ \\
\midrule
\textbf{LATS} w/o Exploration term in node selection mechanism  & 0.8657 & 2.2700 & 12.6108 & 0.8976 & 0.8637 & 18.7534 \\
\textbf{LATS} w/o Online Value Backup   & 1.0877 & 4.0201 & 6.8011 & 1.6677 & 6.1229 & 8.5721 \\
\midrule
\textbf{LATS} & \textbf{0.7223} & \textbf{1.7286} & \textbf{14.0967} & \textbf{0.7588} & \textbf{0.7692} & \textbf{20.1346} \\
\bottomrule
\end{tabular}
\end{table}
\paragraph{Impact of history-aware sampling mechanism of LATS on Diverse Target Discovery}
We ablate the online value backup by disabling value propagation after each observation, yielding a variant we call LATS w/o Online Value Backup. As shown in Table~\ref{tab:preference_alignment_results} (row 2), this variant yields lower QADAS, demonstrating that online value backup at every feedback step is essential. It folds past feedback into the tree’s value estimates, enabling feedback-aware global search that systematically steers future trajectories toward diverse high-value modes while still maintaining quality.
\vspace{-6pt}

\section{Dataset and Training Hyperparameter Details}
\subsection{Datasets}
We construct several long-tailed datasets to evaluate our method. For MNIST-LT and CIFAR10-LT, we follow a fixed exponential long-tail profile using the default class ordering, with class counts
$[5000, 2997, 1796, 1077, 645, 387, 232, 139, 83, 50]$,
which preserves all classes while inducing a strong head-to-tail ratio. For ImageNet, we conduct our experiments on $64 \times 64$ resolution and build a 10-class ImageNet-LT subset with class sizes
\text{dog} (10,000), \text{bird} (6,000), \text{car} (3,600), \text{cat} (2,200), \text{orange} (1,300), \text{shark} (900), \text{frog} (750), \text{shoe} (650), \text{green apple} (550), and \text{bear} (450).
The categories bird, dog, car, and cat are treated as super-classes formed by aggregating multiple fine-grained ImageNet categories within similar semantic group, introducing realistic intra-class diversity. For MICRO2D, we select seven microstructure classes with counts \text{GRF} (10,000), \text{AngEllipse} (6,000), \text{RandomEllipse} (3,600), \text{VoronoiMediumSpaced} (2,000), \text{NBSA} (1,000), \text{VoronoiLarge} (600), and \text{VoidSmall} (300).
\subsection{Model Hyperparameters}
The diffusion backbone follows a UNet-style architecture whose channel widths depend on the dataset scale.
For MNIST-LT, we use block widths (32, 64, 64, 64) with 2 residual blocks. For CIFAR-10-LT, we use block widths (128, 256, 256, 256) with 2 residual blocks.
For ImageNet-LT, we increase capacity to (128, 256, 256, 256, 512) with 2 residual blocks, and for MICRO2D, we use (64, 128, 256, 256, 512) with 2 residual blocks. For MNIST-LT, we use a training batch size of 256, EMA rate 0.99, 1000 training reverse diffusion steps, and 1000 epochs. For CIFAR10-LT and ImageNet-LT, we use a training batch size of 128, an EMA rate of 0.9999, 4000 training reverse diffusion steps, and 1000 epochs. For MICRO2D, we use a training batch size of 128, an EMA rate of 0.9999, 4000 training reverse diffusion steps, and 400 epochs. All of the models use the Adam optimizer. During inference time, all of the models use DDIM sampling with 100 reverse diffusion steps. The reward models are implemented using lightweight classifiers.
For MNIST, we employ a 10-layer CNN, while CIFAR-10, ImageNet-LT, and Mat-LT use ResNet classifiers with depths 18, 34, and 18, respectively, trained on datasets with uniform class distributions.

For each dataset, DLM and DDIM use identical UNet architectures and share all training and diffusion hyperparameters, including the optimizer, the number of optimization steps, EMA rate, and forward and reverse diffusion schedules. All DLMs use $\alpha = 1.7$.

\section{Derivation of Equation~\ref{eq:value}}
  In this section, we provide the detailed derivation for the soft value function recursion (the soft Bellman equation) and the optimal target policy for inference-time alignment. The derivations are provided in~\citep{jain2025diffusion}. For completeness, we have included it here.

\subsection{Soft Value Function (Equation~\ref{eq:value})}

\begin{equation}
\begin{aligned}
    \hat{v}_t(x_t) &= \frac{1}{\lambda} \log \mathbb{E}_{p_\theta(x_{0:t-1} \mid x_t)} \left[ \exp(\lambda r(x_0)) \right] \\
    &= \frac{1}{\lambda} \log \int p_\theta(x_0, \dots, x_{t-1} \mid x_t) \exp(\lambda r(x_0)) \, dx_0 \dots dx_{t-1} \\
    &= \frac{1}{\lambda} \log \int p_\theta(x_0, \dots, x_{t-2} \mid x_{t-1}) p_\theta(x_{t-1} \mid x_t) \exp(\lambda r(x_0)) \, dx_0 \dots dx_{t-1} \\
    &= \frac{1}{\lambda} \log \int p_\theta(x_{t-1} \mid x_t) \underbrace{ \left( \int p_\theta(x_0, \dots, x_{t-2} \mid x_{t-1}) \exp(\lambda r(x_0)) \, dx_0 \dots dx_{t-2} \right) }_{= \exp(\lambda \hat{v}_{t-1}(x_{t-1}))} dx_{t-1} \\
    &= \frac{1}{\lambda} \log \int p_\theta(x_{t-1} \mid x_t) \exp(\lambda \hat{v}_{t-1}(x_{t-1})) \, dx_{t-1} \\
    &= \frac{1}{\lambda} \log \mathbb{E}_{p_\theta(x_{t-1} \mid x_t)} \left[ \exp(\lambda \hat{v}_{t-1}(x_{t-1})) \right].
\end{aligned}
\end{equation}

The above relation combined with the terminal condition $\hat{v}_0(x_0) = r(x_0)$ gives Equation~\ref{eq:value}.
\section{Additional Qualitative Comparisons across Different Settings}

In this section, we present additional qualitative comparisons of generated samples across a range of target sets and datasets, as shown in Figures~\ref{fig:visual_comparisons_all_1} and~\ref{fig:visual_comparisons_all_2}. These visualizations reveal that FKS often degrades sample fidelity due to noisy reward-gradient guidance estimation at intermediate denoising steps, whereas DTS largely preserves visual quality but fails to reliably generate from high-utility tail classes in the underlying distribution. In contrast, LATS consistently produces high-quality samples across all target classes, including rare yet high-utility long-tail categories, demonstrating its ability to balance precise conditioning with broad, utility-aware coverage of the target space.

\begin{figure*}[h]
    \centering

    \begin{subfigure}{\textwidth}
        \centering
        \includegraphics[height=3.8cm,width=0.6\textwidth]{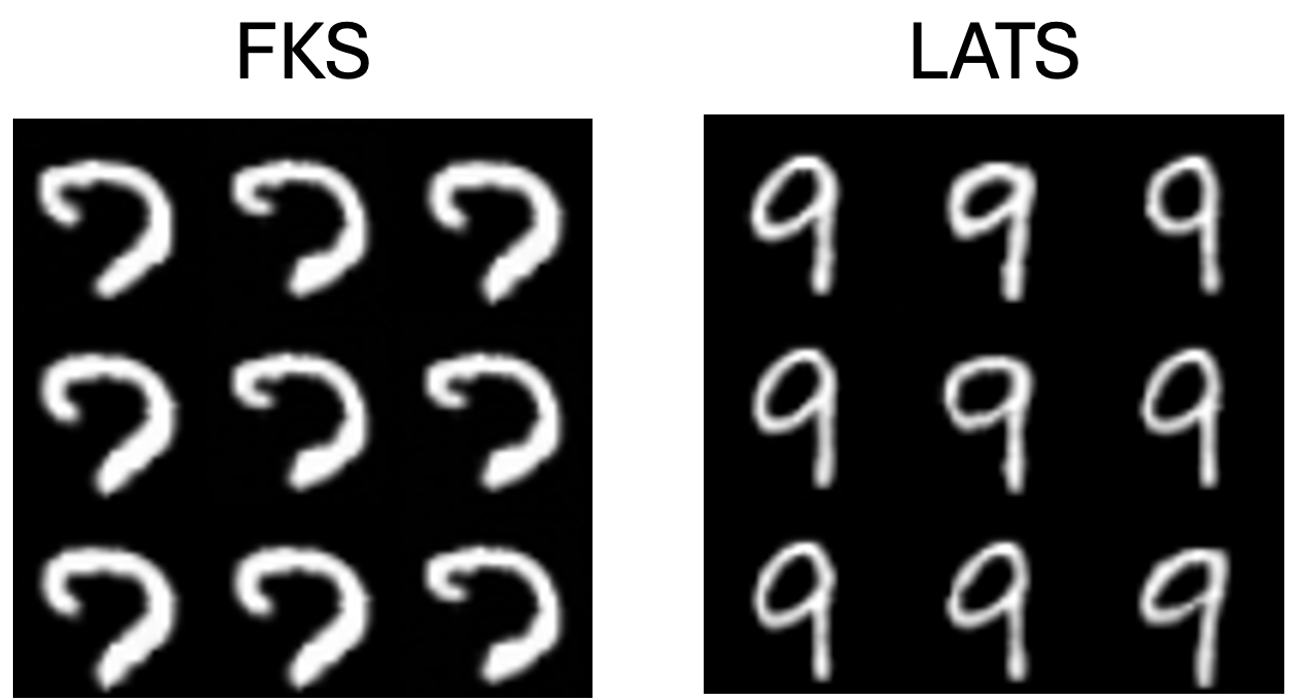}
        \caption{MNIST Target Set \{9\}}
        \label{fig:mnist_s2}
    \end{subfigure}

    \vspace{0.5em}

    \begin{subfigure}{\textwidth}
        \centering
        \includegraphics[height=3.8cm,width=0.6\textwidth]{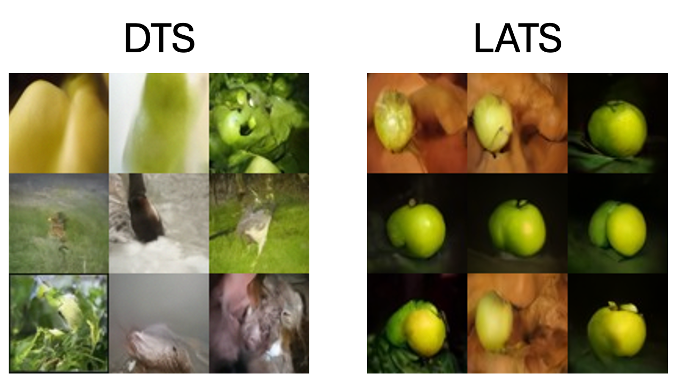}
        \caption{ImageNet with target \{"Green Apple"\}}
        \label{fig:imagenet_apple}
    \end{subfigure}

    \vspace{0.5em}

    \begin{subfigure}{\textwidth}
        \centering
        \includegraphics[width=0.78\textwidth]{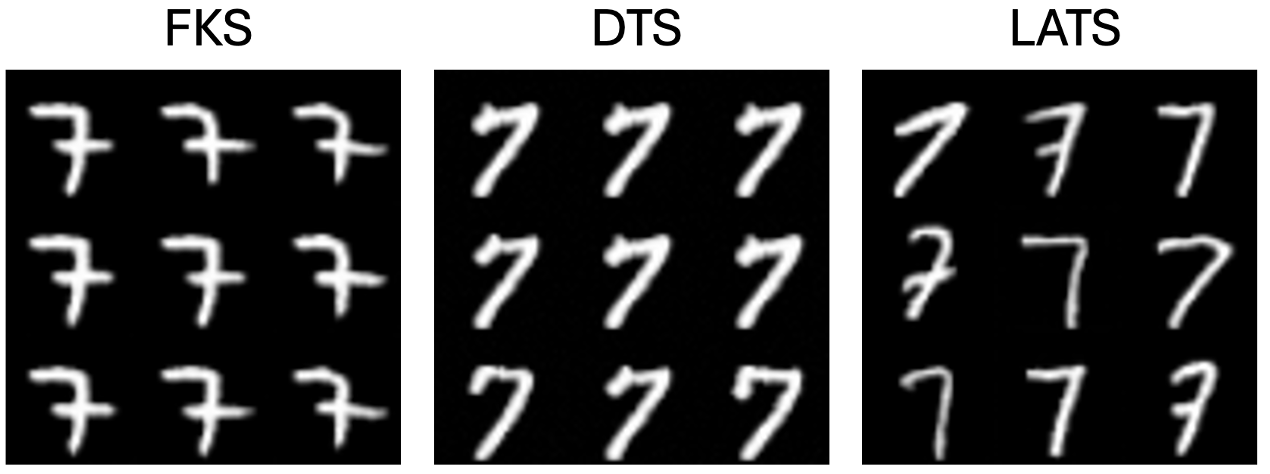}
        \caption{MNIST Target Set \{7\}}
        \label{fig:mnist_7}
    \end{subfigure}

    \vspace{0.5em}

    \begin{subfigure}{\textwidth}
        \centering
        \includegraphics[width=0.78\textwidth]{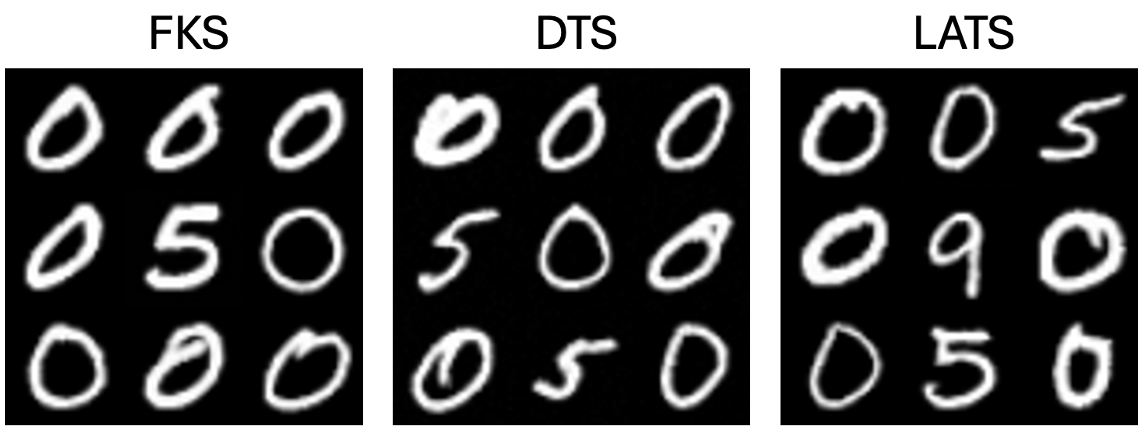}
        \caption{MNIST Target Set \{0, 5, 9\}}
        \label{fig:mnist_059}
    \end{subfigure}

    \vspace{0.5em}

    \caption{\textbf{\small{Qualitative comparison of generation diversity and quality across target sets and datasets (MNIST-LT, ImageNet-LT examples).}}
    \small{We compare \textbf{LATS} against the \textbf{DTS, FKS} baseline on MNIST-LT and ImageNet-LT on representative target sets.}}
    \label{fig:visual_comparisons_all_1}
\end{figure*}

\begin{figure*}[h]
    \centering

    \begin{subfigure}{\textwidth}
        \centering
        \includegraphics[width=0.55\textwidth]{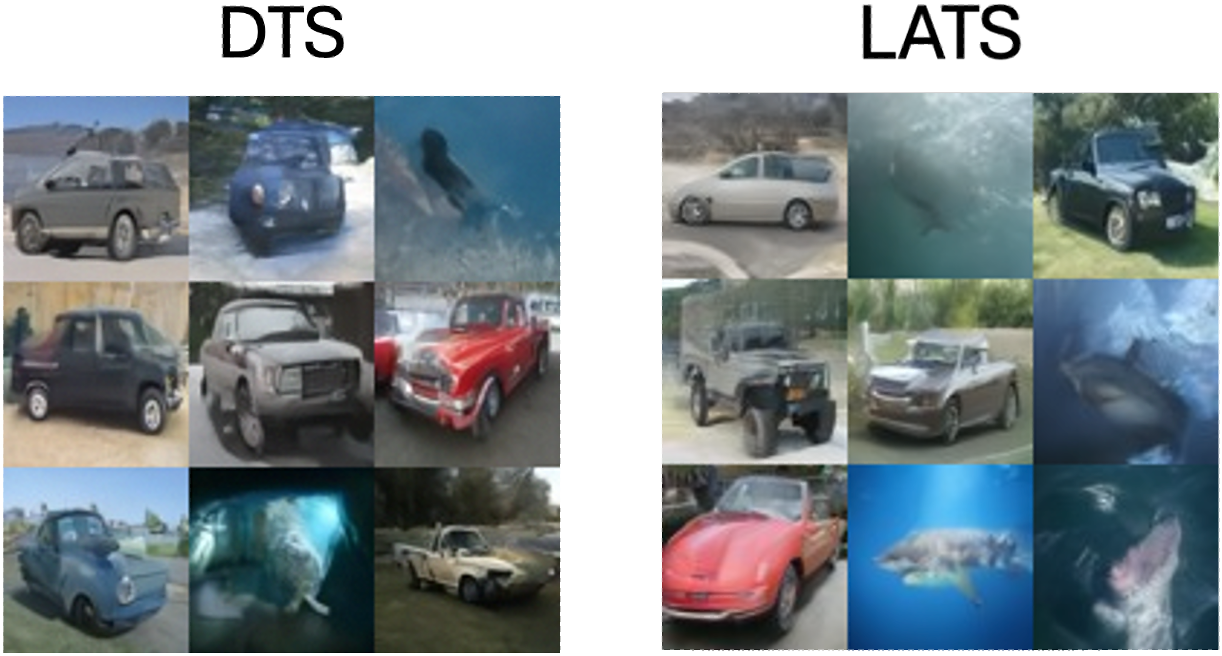}
        \caption{ImageNet Target Set \{"Car", "Shark"\}}
        \label{fig:imagenet_shark_car}
    \end{subfigure}

    \vspace{0.5em}

    \begin{subfigure}{\textwidth}
        \centering
        \includegraphics[width=0.78\textwidth]{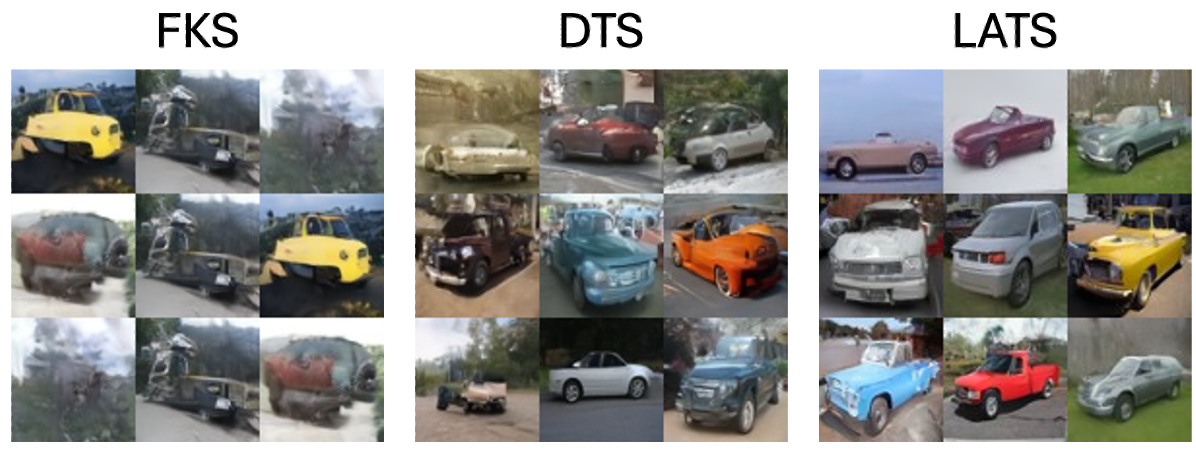}
        \caption{ImageNet Target Set \{"Car"\}}
        \label{fig:imagenet_car2}
    \end{subfigure}

    \vspace{0.5em}

    \begin{subfigure}{\textwidth}
        \centering
        \includegraphics[width=0.78\textwidth]{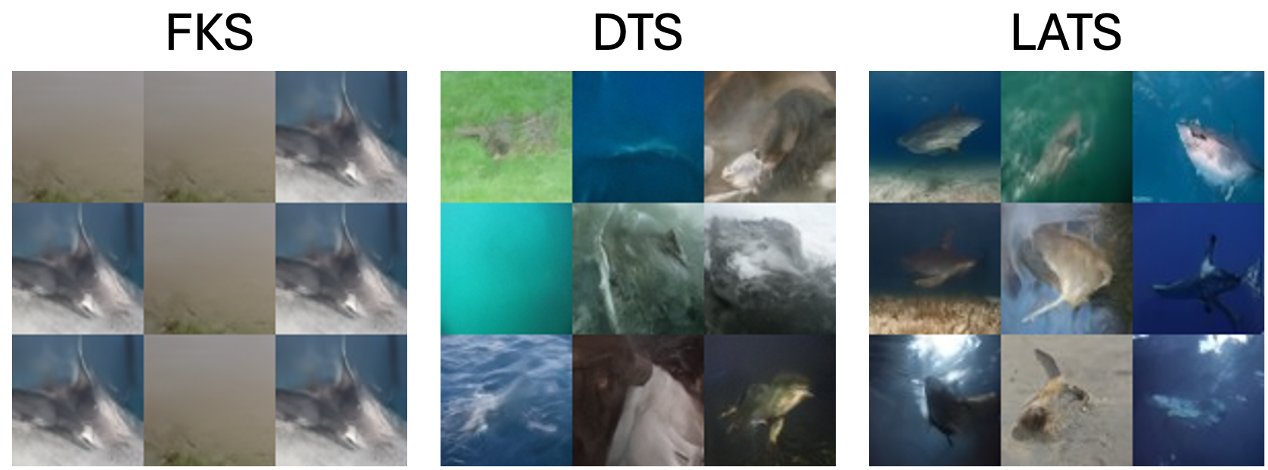}
        \caption{ImageNet Target Set \{"Shark"\}}
        \label{fig:imagenet_shark2}
    \end{subfigure}

    \vspace{0.5em}

    \caption{\textbf{\small{Qualitative comparison of generation diversity and quality across target sets and datasets (ImageNet-LT examples).}}
    \small{We compare \textbf{LATS} against the \textbf{DTS, FKS} baseline on ImageNet-LT on representative target sets.}}
    \label{fig:visual_comparisons_all_2}
\end{figure*}

\end{document}